\documentclass[12pt]{article}

\usepackage{times}
\usepackage{fullpage}
\usepackage{parskip}
\usepackage[utf8]{inputenc} %
\usepackage[T1]{fontenc}    %
\usepackage[
  backref=true,
  backend=biber,
  natbib=true,
  style=alphabetic,
  sorting=alphabeticlabel,
  sortcites=true,
  minbibnames=3,
  maxbibnames=999,
  mincitenames=4,
  maxcitenames=4,
  minalphanames=4,
  maxalphanames=4,
  url=false,
  doi=false,
]{biblatex}

\DeclareSortingTemplate{alphabeticlabel}{
  \sort[final]{%
    \field{labelalpha} %
  }
  \sort{%
    \field{year}   %
  }
  \sort{%
    \field{title}  %
  }
}

\AtBeginRefsection{\GenRefcontextData{sorting=ynt}}
\AtEveryCite{\localrefcontext[sorting=ynt]}
\usepackage[hypertexnames=false,hidelinks]{hyperref} %
\hypersetup{
  pdftitle={Understanding Private Evolution as Learning-Augmented Clustering},
  pdfauthor={Audra McMillan, Kunal Talwar, Felix Zhou},
}
\usepackage{url}            %
\usepackage{xurl}           %
\usepackage{booktabs}       %
\usepackage{amsfonts}       %
\usepackage{nicefrac}       %
\usepackage{microtype}      %
\usepackage[dvipsnames]{xcolor}         %
\usepackage{amsmath}
\usepackage{amssymb}
\usepackage{mathtools}
\usepackage{amsthm}
\usepackage[capitalise,noabbrev,nameinlink]{cleveref}
\usepackage{graphicx}
\usepackage{wrapfig}

\usepackage{thmtools}
\usepackage{thm-restate}
\usepackage{subcaption}
\usepackage[ruled,vlined,linesnumbered,algosection,nosemicolon]{algorithm2e}
\usepackage[shortlabels,inline]{enumitem}
\usepackage{float}
\usepackage{comment}
\usepackage{xspace}
\usepackage{dsfont}
\usepackage{mathrsfs}

\newtheorem{theorem}{Theorem}[section]

\newtheorem{assumption}{Assumption}
\newtheorem{definition}[theorem]{Definition}
\newtheorem{lemma}[theorem]{Lemma}
\newtheorem{proposition}[theorem]{Proposition}
\newtheorem{corollary}[theorem]{Corollary}
\newtheorem{example}[theorem]{Example}

\AddToHook{env/theorem/begin}{\crefalias{section}{theorem}}
\AddToHook{env/fact/begin}{\crefalias{section}{fact}}
\AddToHook{env/assumption/begin}{\crefalias{section}{assumption}}
\AddToHook{env/definition/begin}{\crefalias{section}{definition}}
\AddToHook{env/lemma/begin}{\crefalias{section}{lemma}}
\AddToHook{env/proposition/begin}{\crefalias{section}{proposition}}
\AddToHook{env/corollary/begin}{\crefalias{section}{corollary}}
\AddToHook{env/example/begin}{\crefalias{section}{example}}
\AddToHook{env/remark/begin}{\crefalias{section}{remark}}

\crefname{section}{Section}{Sections}
\crefname{theorem}{Theorem}{Theorems}
\crefname{fact}{Fact}{Facts}
\crefname{assumption}{Assumption}{Assumptions}
\Crefname{assumption}{Assumption}{Assumptions}
\crefname{definition}{Definition}{Definitions}
\crefname{lemma}{Lemma}{Lemmas}
\crefname{proposition}{Proposition}{Propositions}
\crefname{corollary}{Corollary}{Corollaries}
\crefname{example}{Example}{Examples}
\crefname{remark}{Remark}{Remarks}
\crefname{equation}{Equation}{Equations}
\crefname{appendix}{Appendix}{Appendices}
\crefname{algorithm}{Algorithm}{Algorithms}
\crefname{algocf}{Algorithm}{Algorithms}
\crefname{figure}{Figure}{Figures}

\newcommand{\eg}{\emph{e.g.}}
\newcommand{\ie}{\emph{i.e.}}
\newcommand{\iid}{\emph{i.i.d.}}

\newcommand{\numSamples}{m}
\newcommand{\subsamplePE}{\texttt{SubsamplePE}}
\newcommand{\augPE}{\texttt{AugPE}}

\newcommand{\inbrace}[1]{\ensuremath{\left\{#1\right\}}}

\newcommand{\inparen}[1]{\ensuremath{\left(#1\right)}}
\newcommand{\insquare}[1]{\ensuremath{\left[#1\right]}}

\newcommand{\R}{\mathbb{R}}

\newcommand{\E}{\mathbb{E}}

\newcommand{\sset}{\subseteq}

\newcommand{\mcal}{\mathcal}

\newcommand{\nfrac}{\nicefrac}

\newcommand{\GAPE}{{\sc GAPE}}

\DeclarePairedDelimiter{\card}{\lvert}{\rvert}
\DeclarePairedDelimiter{\abs}{\lvert}{\rvert}
\DeclarePairedDelimiter{\norm}{\lVert}{\rVert}
\DeclarePairedDelimiter{\set}{\lbrace}{\rbrace}
\DeclarePairedDelimiter{\iprod}{\langle}{\rangle}

\DeclareMathOperator{\argmin}{argmin}

\DeclareMathOperator{\Var}{Var}
\DeclareMathOperator{\supp}{supp}

\newcommand{\eps}{\ensuremath{\varepsilon}}

\DeclareMathOperator{\Cov}{Cov}

\DeclareMathOperator{\proj}{proj}
\DeclareMathOperator{\RanAPI}{RandomAPI}
\DeclareMathOperator{\VarAPI}{VariationAPI}
\DeclareMathOperator{\NNhist}{NNhistogram}
\DeclareMathOperator{\Selection}{Selection}
\DeclareMathOperator{\ddim}{ddim}
\DeclareMathOperator{\tr}{tr}
\newcommand{\dd}{\mathop{}\!\mathrm{d}}
\DeclareMathOperator{\Lap}{Lap}
\DeclareMathOperator{\Rad}{Rad}
\DeclareMathOperator{\BL}{BL}

\DeclareMathOperator{\LSCost}{LSObj}
\DeclareMathOperator{\range}{range}

\NewDocumentEnvironment{pf}{o}
  {\IfNoValueTF{#1}{\begin{proof}}{\begin{proof}[Proof of #1.]}}
  {\IfNoValueTF{#1}{\end{proof}}{\end{proof}}}

\newcommand{\myparagraph}[1]{\paragraph{#1}}
\newcommand{\tpdp}{0}

\title{Understanding Private Evolution as Learning-Augmented Clustering}

\author{%
  Audra McMillan \\
  Apple
  \and
  Kunal Talwar \\
  Apple
  \and
  Felix Zhou\thanks{Part of this work was completed during an internship at Apple Research.} \\
  Yale University
}
\date{}

\begin{document}

\maketitle

\begin{abstract}
   Private Evolution (PE) is a differentially private algorithm for synthetic data generation. While it can be viewed as a Wasserstein learning algorithm, it performs much better in practice than worst‑case Wasserstein analyses would predict. We recast PE as generative model‑augmented Wasserstein learning. We show theoretically that when we take into account the use of a generative model that is able to capture something about the true distribution, then we can obtain much better performance bounds. For example, if the generator gives samples in the same low-dimensional space as the distribution, then sample complexity depends on intrinsic, not ambient, dimension. We also show that standard variants of PE can fail to converge on simple well-clustered instances, and propose a new geometry-aware version of PE with provable convergence on such instances. Experimentally, we show that our new algorithm is competitive with standard baselines and can improve recall.
\end{abstract}

\section{Introduction} 

When interacting with sensitive data, differentially private algorithms allow for accurate analysis and learning while providing a provable privacy guarantee. A differentially private (DP) synthetic data set is designed to capture various properties of a real dataset and can be useful for answering queries, as well as for evaluating and training models. DP synthetic data can avoid the computational and complexity overheads of differentially private algorithms for the same tasks. Additionally, a private synthetic dataset can be reused as often as needed, without having to worry about privacy budgets. This has led to significant interest in designing algorithms for, and understanding the fundamental limits of, private synthetic data generation.

Foundation models have recently been useful for a wide array of tasks, and have recently been used as a tool for synthetic data generation. \citet{lin2024images} recently proposed a genetic algorithm for generating private synthetic data known as Private Evolution (PE). This elegant algorithm (see \cref{alg:private-evolution-template}) interacts with the sensitive data through a very simple interface, making it suitable for both centralized and federated settings. Given access to an appropriate foundation model through a {\em VariationAPI}, it generates remarkably good synthetic data in both the image and text domains. This has led to a flurry of research aimed at understanding and improving the algorithm.

The private synthetic data problem can be naturally formalized as distribution learning in Wasserstein distance. This formulation has been explored in many recent works, including those analyzing private evolution. These analyses however give sample complexity bounds that are exponential in the dimension $d$, and this is in general unavoidable: any algorithm for Wasserstein learning in $d$ dimensions provably needs exponential in $d$ sample complexity in the worst case~\citep{singh2018minimax,dudley1969speed}. When used with embeddings of text or data as defining the metric, these minimax bounds are nearly uninformative. Even with embedding dimension as small as 128, one would need roughly $2^{128}$ times as many samples to reduce the Wasserstein error by a factor of two.

In this work, we provide a fresh viewpoint on the PE algorithm, through the lens of beyond worst case analysis. We study the algorithm in two frameworks beyond the worst case. As the foundation model API is crucial in the success of PE, we argue that any analysis of PE must exploit properties of the foundation model. We phrase the problem in the {\em learning augmented algorithms} framework, where our Wasserstein learning algorithm is analyzed under assumptions about the foundation model.

A natural assumption one can make is that generator distribution is (locally) low-dimensional. Indeed 
\citet{lin2024images} showed that the variation API output has effective dimensionality ($\approx 150$) that is significantly smaller than the ambient dimension for images. We empirically show that for typically-used generators for the text domain, the effective dimension is even smaller ($\approx 15$; \cref{apx:intrinsic-dimensionality}). For the generator to be useful, it needs to be ``compatible'' with the distribution being learnt. We measure this compatibility in terms of the relative variance of the VariationAPI in the direction of the private data (\cref{def:reachable}). In \Cref{sec:low-dim-var-api}, we show that these assumptions allow us to circumvent these worst-case rates to obtain an $\tilde O(\lambda^{1+\frac1{2k}} n^{-\frac1{k}})$ rate of convergence, where $\lambda, k$ are two measures of ``intrinsic dimensionality'' of the $\VarAPI$.

\begin{theorem}[Informal; see \Cref{thm:private-evolution-worst-case}]\label{thm:private-evolution-worst-case-informal}
    Suppose the $\VarAPI$ is supported on a set of doubling dimension at most $k\geq2$ and has effective rank $\lambda$ with respect to the dataset.
    Then \Cref{alg:private-evolution-worst-case} (a standard instantiation of PE) is $(\eps, \delta)$-DP and outputs a synthetic dataset with expected $W_1$-error
    \[
        \tilde O\inparen{\lambda^{1+\frac1{2k}} \inparen{\frac{\sqrt{\log(\nicefrac{1.25}\delta)}}{\eps n}}^{\frac1{k}} }\,.
    \]
\end{theorem}

This result only depends on the intrinsic dimensionality, and is independent of the ambient dimension! This is in sharp contrast to what is achievable without a generator: it is easy to show that even when the distribution of interest is supported on some $1$-dimensional subspace, any Wasserstein learning algorithm for privately learning this distribution (under concentrated DP) needs at least $\sqrt{d}$ samples when the ambient dimension is $d$. This implies that one cannot hope to get bounds independent of the ambient dimension without access to (and compatibility assumptions about) the generator. We thus give the first theoretical explanation for the importance of the generator in PE.

This analysis still has the Wasserstein error falling as $n^{-\frac1{k}}$, when the intrinsic dimensionality is $k$. While this bound is significantly better than the $n^{-\frac 1 d}$ dependence previously known, it can be a relatively slow rate when the intrinsic dimensionality $k$ is moderately large. Such a dependence is provably unavoidable in the worst-case, and we turn our attention to simple models for the data distribution to better understand the algorithm.

Inspired by \citet{lin2024images} who studied a duplicated data model, we study a simple clustered data model, where the data distribution is supported on a small number of well-separated clusters. While we would expect any reasonable algorithm to perform well on such instances, we surprisingly show that known variants of PE \citep{lin2024images, xie2024text} can badly fail even on these very simple instances. 
To the best of our knowledge, we are the first to identify this failure mode.

We identify the Selection step (where the algorithm selects the next iteration of the synthetic data from a set of candidates generated by the $\VarAPI$) in PE to be the main reason for this: both random selection and rank-based selection can miss a constant fraction of the clusters. We argue that the selection process should be made geometry-aware, and propose such a new variant of PE that modifies the selection step in a careful way. We show that when the clusters are well-separated and sufficiently large, then our algorithm \GAPE\ ({\bf G}eometrically {\bf A}ware {\bf P}rivate {\bf E}volution) provably recovers a good approximation to the private data.

\begin{theorem}[Informal; see \Cref{thm:private-evolution-clustering}]\label{thm:private-evolution-clustering-informal}
    Suppose the dataset consists of $\kappa$ clusters of diameter $r$ with intercluster distance $R>0$.
    There are such instances on which PE with ranking or subsampling outputs $S$ with coverage error $\max_{x\in D}\rho(x,S)=\Omega(R)$.
    On the other hand, suppose that the $\VarAPI$ satisfies the reachability condition with effective rank $\lambda$, that $R\gg\lambda r$, and that the minimum cluster size is
    \[
        \tilde\Omega\inparen{\lambda^{\frac14} \sqrt{\frac{n\log(\nicefrac{1.25}\delta)}\eps}}\,,
    \]
    then, under \Cref{asmp:warm-start}, for fixed $c$, \GAPE\ (\Cref{alg:private-evolution-clustering}) is $(\eps, \delta)$-DP and outputs $S$ with coverage error $O(\lambda r)$ with high probability.
\end{theorem}

Finally, in \Cref{sec:experiments}, we empirically evaluate \GAPE\ on text datasets to understand if the algorithm's provable performance in synthetic models translates to improvements on real-world data that may not be as well-clustered.
We show that when the input dataset size is significantly larger than the number of synthetic datapoints being generated, \GAPE\ is competitive with existing variants of PE, with improvements that depend on the dataset, metric, and privacy budget.

In summary, we make the following contributions:
\begin{itemize} %
    \item We bring a beyond-worst-case analysis viewpoint to understanding PE. We give a significantly improved analysis of PE for Wasserstein Learning, in the learning-augmented framework. Our work shows that access to the generator provably helps in learning. We empirically study the low dimensionality of typically-used text generators.
    \item We show that in a simple clustered data model, existing PE variants can provably underperform. This motivates us to propose a new algorithm (\GAPE) and we show that it provably learns well-clustered data.
    \item Our empirical results show that \GAPE\ is competitive with existing PE variants, and improves recall in several settings.
\end{itemize}

\section{Preliminaries}

An outline of related works in this area can be found in Appendix~\ref{relatedworks}. A notation list can be found in Appendix~\ref{notation}. Differential privacy background can be found in Appendix~\ref{DPprelims}.

\subsection{Private Evolution Framework}
The PE framework assumes access to two generative APIs. $\RanAPI$ generates unconditional samples from some underlying distribution. This models a foundation model which may be prompted with some data-independent prompt.
 $\VarAPI$ takes as input data and generates variations conditioned on this input. This models a foundation model that is asked to generate variations of the input data.
PE begins with a random dataset obtained from a $\RanAPI$. It then iteratively improves the synthetic dataset, moving it closer to the private data, over $T$ rounds.
At each iteration, it generates variations using the $\VarAPI$ with the hope that some subset of the variations is a better synthetic dataset than the one from the previous round. 
Each true data point votes for the variation most like it (its nearest variation in the embedding space) to generate the nearest neighbor histogram. Gaussian noise is added to achieve differential privacy.
Finally, a $\Selection$ mechanism is applied to the noisy histogram to choose the next iteration of synthetic data.

Any distance metric $\rho$ can be used in the nearest neighbor computation, although, as in prior work, we will focus on the setting where $\rho$ is the Euclidean distance in an embedding space. That is, if $f$ is an embedding that maps data points into a Euclidean embedding space $\mathbb{R}^d$, then $\rho(x,y)=\|f(x)-f(y)\|_2$.
For simplicity of notation, throughout the remainder of this paper, we will assume that all computation occurs inside the embedding space $\mathbb{R}^d$. That is, we will assume that $\RanAPI$ outputs elements in $\mathbb{R}^d$ and $\VarAPI$ takes as input and output elements in $\mathbb{R}^d$. This is without loss of generality, assuming the existence of an embedding model that projects the data into $\mathbb{R}^d$. Thus, the ambient dimension $d$ is understood to be the embedding dimension. 

In \Cref{alg:private-evolution-template} we give pseudocode for the general outline of PE studied in prior works~\citep{lin2024images,xie2024text,gonzalez2025private}.
The $\RanAPI$, $\VarAPI$, and selection mechanism are the components studied below.

\newcommand{\mycommentstyle}[1]{\textcolor{gray}{#1}}
\SetCommentSty{mycommentstyle}

\SetKwComment{Comment}{$\triangleright$\ }{}

\SetKwProg{Fn}{Function}{}{end}

\SetKwFunction{privateEvolutionTemplate}{privateEvolutionTemplate}
\begin{algorithm}[htp]
    \caption{Private Evolution Template\label{alg:private-evolution-template}}
    \KwData{private dataset $D$; privacy parameters $\varepsilon$, $\delta$; number of synthetic samples $m$; iterations $T$; generative APIs $\RanAPI$, $\VarAPI$; distance metric $\rho$; selection mechanism $\Selection$}
    \Fn{\privateEvolutionTemplate{D}}{
        \BlankLine
        $\sigma \gets \frac{2\sqrt{T \log(\nfrac{1.25}\delta)}}{\eps}$\;
        \Comment{noise magnitude for privacy}
        \BlankLine
        $S^{(0)} \gets \RanAPI(m)$\;
        \For{$t=1, \dots, T$}{
            $V^{(t)} \gets \VarAPI(S^{(t-1)})$\;
            $H^{(t)} \gets \NNhist(D, V^{(t)}, \rho)$\;
            \Comment{nearest-neighbor histogram (\Cref{alg:nearest-neighbors-histogram})}
            $\tilde H^{(t)}\gets H^{(t)} + \mcal N(0, \sigma^2 I_{\card{V^{(t)}}})$\;
            $S^{(t)}\gets \Selection(\tilde H^{(t)}, \numSamples)$\;
            \Comment{abstract selection algorithm}
        }
        \Return $S^{(T)}$\;
	}
\end{algorithm}

$\NNhist$ builds a nearest neighbor histogram where each real data point votes for the synthetic data point closest to it, as shown in \Cref{alg:nearest-neighbors-histogram}.

\SetKwFunction{NNhistogram}{NNhistogram}
\begin{algorithm}[H]
    \caption{Nearest Neighbors Histogram\label{alg:nearest-neighbors-histogram}}
    \Fn{\NNhistogram{$D$, $V$, $\rho$}}{
        $H\gets [0, 0, \dots, 0]$\;
        \For{$x\in D$}{
            $\bar v\gets \argmin_{v\in V} \rho(x, v)$\;
            $H[\bar v] \gets H[\bar v] + 1$\;
        }
        \Return $H$\;
	}
\end{algorithm}

In \Cref{alg:private-evolution-worst-case}, we present pseudocode for a particular instantiation PE \citet{gonzalez2025private}. This is the version of PE that we will study in \Cref{sec:low-dim-var-api}. %
We write $\VarAPI(y)$ as a shorthand for $\bigcup_{\ell=0}^{\ell^\star} \VarAPI(y, 2^{-\ell})\cup \set{y}$
to denote the set of all generated variations
at different scales, including the original synthetic point, with scale $1$ and smallest scale $2^{-\ell^\star}$.
We also use the notation $\VarAPI(S) \coloneqq \cup_{y\in S} \VarAPI(y)$ in the natural way.
In this instantiation, the selection mechanism is implemented using a bounded-Lipschitz (BL) projection and subsampling. More details on the BL distance are given in \Cref{sec:proofoutline}. 

\SetKwFunction{privateEvolutionBeyond}{privateEvolutionBeyond}
\begin{algorithm}[htp]
    \caption{Beyond Worst-Case Private Evolution\label{alg:private-evolution-worst-case}}
    \Fn{\privateEvolutionBeyond{$D$, $T$, $\eps$, $\delta$, $k$, $\lambda$, $c$}}{
        \BlankLine
        $\sigma \gets \frac{2\sqrt{T \log(\nfrac{1.25}\delta)}}{n\eps}$\;
        \BlankLine
        $m\gets\lceil\sigma^{-1}\rceil$, $\alpha\gets\min\{1,\sigma^{1/k}\}$\;
        $\ell^\star\gets\lceil\log_2(200\lambda/(c\alpha^2))\rceil$\;
        \Comment{set $2^{-\ell^\star}$ to be the smallest $\VarAPI$ scale}
        $\VarAPI(\cdot) \gets \bigcup_{\ell=0}^{\ell^\star} \VarAPI(\cdot, 2^{-\ell})\cup \set{\cdot}$\;
        \Comment{For each previous candidate, generate $1$ variation at each of $\ell^\star+1$ scales}
        $S^{(0)} \gets \RanAPI(m)$ \Comment{arbitrary initialization}
        \For{$t=1, \dots, T$}{
            $V^{(t)} \gets \bigcup_{y\in S^{(t-1)}}\VarAPI(y)$\;
            $H^{(t)} \gets \frac1n\NNhist(D, V^{(t)}, \rho)$\;
            $\tilde H^{(t)}\gets H^{(t)} + \mcal N(0, \sigma^2 I_{\card{V^{(t)}}})$\;
            $\bar H^{(t)}\gets \proj_{\BL}(\tilde H^{(t)})$\;
            \Comment{Compute BL projection back to simplex}
            $S^{(t)}\sim_{\iid{}} (\bar H^{(t)})^{\otimes m}$\;
            \Comment{Subsample $m$ elements for next iteration}
        }
        \Return $S^{(T)}$\;
	}
\end{algorithm} 

\section{VariationAPI with Low Effective Rank \& Doubling Dimension}\label{sec:low-dim-var-api}

In this section, we show that exploiting the intrinsic low dimensionality of the $\VarAPI$ allows us to circumvent worst-case lower bounds for Wasserstein learning.
We identify two notions of dimensionality that can be used to provide convergence guarantees for private evolution with subsampling in terms of $W_1$-distance. The new convergence guarantees are independent of the ambient dimension. If $\VarAPI$ is low-dimensional then these new bounds are substantially better than the worst-case lower bounds, which are exponential in the ambient (embedding) dimension. 
Prior works~\cite{gonzalez2025private,lin2024images} modeled $\VarAPI$ as an isotropic Gaussian distribution with tunable variance.
We expand this and model $\VarAPI$ as a general distribution in $\mathbb{R}^d$.
The contraction argument (\Cref{lem:contraction-one-scale}) includes the specific Gaussian $\VarAPI$ as a special case.

\paragraph{\Cref{asmp:doubling-space}: Low Doubling Dimension}
Prior to stating the main result, we introduce the specific notions of ``low-dimensionality'' studied in this work.
This requires formalizing the definition of the $\VarAPI$.
We assume sample access to two distribution classes $\RanAPI$ and $\VarAPI(y, s^2)$ for $y\in \Omega$ and $s > 0$.
For this section, we focus mainly on $\VarAPI$ and assume that $\RanAPI$ samples from some distribution supported on $\Omega$.
We explore the impact of $\RanAPI$ (the initialization conditions) in \Cref{sec:selection-clustering}.
The distribution $\VarAPI = \VarAPI(y, s^2)$ is modeled with parameters $y, s^2$ that specify the input point and the variation scale.
\begin{definition}\label{def:variation-api}
    Let $Z\sim\VarAPI(y,s^2)$ denote a distribution
    supported on $\Omega$, parameterized by $y\in\Omega$
    and $s>0$, satisfying\footnote{We can relax the second condition to $\tr\Cov[Z]= \Theta(s^2)$.}
    $\E[Z]=y$ and $\tr\Cov[Z]=s^2$.
\end{definition}
The parameter $y$ is the data point for which we are producing variations. The parameter $s^2$ is inspired by the temperature parameter of large language models (LLMs). As we increase $s$, the variations being produced become more diverse.

We now introduce the first notion of low-dimensionality, which is related to the support of the $\VarAPI$. 
Recall that a metric space has doubling dimension $k$ if for all $r>0$, every ball of radius $r$ is contained in $2^k$ balls of radius $\nicefrac{r}2$.
\begin{assumption}\label{asmp:doubling-space}
    $(\Omega, \rho=\norm{\cdot}_2)$ is a metric space which has doubling dimension $\ddim(\Omega) = k\leq d$.
\end{assumption}
Intuitively, the doubling dimension controls the number of points needed to approximate a distribution supported on $\Omega$ in $W_1$-distance at a given resolution.
We will only need \Cref{asmp:doubling-space} for our general beyond worst-case result in this section (\Cref{sec:low-dim-var-api}). It is not required for \Cref{sec:selection-clustering}.

\paragraph{\Cref{asmp:reachable}: Reachability}
In order to ensure that PE makes progress in every iteration,
we need to ensure that $\VarAPI$ is likely to output variants that are closer to the private data than the synthetic data from the previous round. That is, $\VarAPI$ puts sufficient mass in the direction of the data.
We formalize this intuition in \Cref{asmp:reachable} below.
\begin{definition}[Reachable]\label{def:reachable}
    We say that the point $x\in \Omega$ is \emph{reachable
    from a distribution $Q$ with effective rank $\lambda\geq1$} if
    for $Z\sim Q$ and $y\coloneqq \E[Z]$, with $x\neq y$ and $(x-y)^\top\Cov[Z](x-y)>0$,
    \begin{enumerate}[noitemsep,topsep=0em]
        \item  
        $
            \frac{\tr\Cov[Z]\cdot \norm{x-y}_2^2}{(x-y)^\top \Cov[Z] (x-y)}
            \leq \lambda\,;
        $ (effective rank)
        \item $\Pr_{Z\sim Q}\left[ (x-y)^\top (Z-y) \geq \sqrt{c\cdot \Var[(x-y)^\top Z]} \right] \geq \frac13$ for an absolute constant $0<c\leq2$.\footnote{Any constant strictly less than $\nicefrac12$ is reasonable in place of $\nicefrac13$.} (anticoncentration)
    \end{enumerate}
\end{definition}

The first condition in \Cref{def:reachable} states that the one dimensional variance of $Q$ in the direction of $x$ is at least a $1/\lambda$ fraction of the total variance $\tr\Cov[Z]$. If $Q$ was isotropic in $\mathbb{R}^d$ then we would have $\lambda=d$, so $\lambda$ acts as an ``effective rank". The second condition is an anticoncentration bound on $Q$ in the direction of $x$, with constant probability, a sample moves toward $x$ by an amount comparable to the standard deviation.

\begin{assumption}\label{asmp:reachable}
    Fix distinct points $x,y\in\Omega$, $\lambda\geq1$,
    and an absolute constant $0<c\leq2$.
    For every $s>0$, $x$ is reachable from $\VarAPI(y,s^2)$
    with effective rank $\lambda$ and anticoncentration constant $c$.
\end{assumption}

\begin{wrapfigure}{r}{0.4\textwidth}
    \centering
    \includegraphics[width=\linewidth]{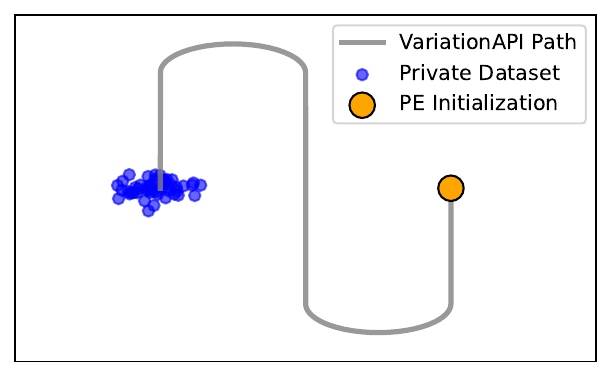}
    \caption{A pathological example of $\VarAPI$ output space.}
    \label{fig:snaking-path}
\end{wrapfigure}

\Cref{asmp:reachable} is required for PE to make progress at each iteration.
\Cref{ex:snake} describes a scenario where this condition fails, and PE fails to converge.

\begin{example}\label{ex:snake}
    \Cref{fig:snaking-path} depicts a dataset and PE initialization, along with a plausible output space of the $\VarAPI$; a one-dimensional manifold (a line) in $\mathbb{R}^2$. Given a synthetic data point $y$, $\VarAPI(y,s^2)$ outputs variants that are close to $y$ in the geodesic distance. \Cref{asmp:reachable} does not hold as $\VarAPI$ does not put mass in the direction of the dataset for most points $y$.
    One may hope that PE would walk along the $\VarAPI$ path and eventually output something close to the private dataset. However, PE can actually stall at a local optimum. Suppose that the PE initialization is at the end of a vertical segment (as in \Cref{fig:snaking-path}) and the variations output by $\VarAPI$ all lie along the same vertical segment as $y$. Although some variants will be closer to the private data in geodesic distance, they are all farther in the Euclidean metric used by PE; so $y$ is chosen and no progress is made.
    
\end{example}

\Cref{ex:reachability-2d} gives a simple instance where \Cref{asmp:reachable} does not hold and PE provably fails.

A promising future direction is to replace the Euclidean metric in the nearest neighbor histogram with a more sophisticated learned metric, which may help avoid the Euclidean local minima illustrated in \Cref{fig:snaking-path}.
Since the manifold on which the generator moves is a function of the generator, such a metric may be learnable without accessing the private dataset.

\subsection{Main Theorem}

Now, we are ready to state the main theorem of this section. Let $\mu_D$ denote the empirical distribution of $D$ and $\mu_S^{(T)}$ denote the empirical distribution of the synthetic dataset after $T$ iterations of \Cref{alg:private-evolution-worst-case}.

\begin{restatable}{theorem}{thmPrivateEvolutionWorstCase}\label{thm:private-evolution-worst-case}
    Fix a domain $\Omega\sset \R^d$ with diameter $1$ and doubling dimension at most $k\geq2$ (\Cref{asmp:doubling-space}).
    Let $D\sset \Omega$ and suppose that, at each iteration with synthetic dataset $S$,
    \Cref{asmp:reachable} holds with parameters $\lambda,c$
    for every $x\in D\setminus S$ and some
    $y_x\in\argmin_{y\in S}\rho(x,y)$.
    If
    $
        T\geq \frac{200\lambda}{c} \log\inparen{\frac{c}{200\lambda \sigma^{\frac1{k}}}}$ then
        $\E\insquare{ W_1\inparen{\mu_S^{(T)}, \mu_D} }
        = \tilde O\inparen{\lambda \sigma^{\frac1{k}}}\,
    $
    where the expectation is over the randomness of the algorithm.
    Further, for $\eps,\delta\in(0,1)$, if $\sigma = \frac{2\sqrt{T \log(\nicefrac{1.25}{\delta})}}{n\eps}$ and
    $
        T =  \left\lceil\frac{200\lambda}{c}\log(2+n\eps)\right\rceil \,,
    $
    then \Cref{alg:private-evolution-worst-case} is $(\eps, \delta)$-DP and $\E\insquare{ W_1\inparen{\mu_S^{(T)}, \mu_D} }$ is at most
    \begin{align*}
        \tilde O\inparen{\lambda^{1+\frac1{2k}} \inparen{\frac{\sqrt{\log(2+n\eps)\log(\nicefrac{1.25}\delta)}}{\eps n}}^{\frac1{k}} }\,.
    \end{align*}
\end{restatable}

The first thing to note about this bound is that it only depends on the intrinsic dimension $k$ and the effective rank $\lambda$, \emph{not} the ambient dimension $d$.
The bound is approximately
\[
    \tilde{O}\left(\lambda^{1+1/(2k)}
    \left(\frac{\sqrt{\log(1.25/\delta)}}{\eps n}\right)^{1/k}\right).
\]
While this still scales as $n^{-1/k}$ for intrinsic dimension $k$, it is a significant improvement over the $n^{-1/d}$ worst case scaling for Wasserstein learning.

There are several other notions of low-dimensionality that are studied in data analysis. In particular, it is commonly conjectured that high dimensional datasets that occur in practice actually lie on low-dimensional manifolds embedded in the ambient space\footnote{This is known as the ``Manifold Hypothesis".}. We conjecture that our results extend to the setting of low-dimensional manifolds with low curvature.

\subsection{Proof Outline}\label{sec:proofoutline} 
We defer the full proof of \Cref{thm:private-evolution-worst-case} to \Cref{apx:private-evolution-worst-case} and sketch the high-level details below.
The key idea is that if $x$ is reachable from $\VarAPI(y, s^2)$, then the distance between $x$ and a single variation at scale $s^2\approx \rho(x, y)^2/\lambda$ is at most $(1-\nfrac{c}{50\lambda})\rho(x, y)$ with constant probability (\Cref{lem:contraction-one-scale}).
In other words, the set of variations will contain a point that approaches the true underlying dataset by some multiplicative factor.
The catch is that we do not have knowledge of $\rho(x, y)$ but can approximate the correct value by producing a variation at each scale (powers of two).
We provide more background information followed by a more detailed exposition below.
This reduces the $W_1$-distance between the true dataset and the set of variations.
The remaining steps are to control the distortion in $W_1$-distance due to Gaussian noise necessary for privacy and subsampling to reduce the size of the support.

Given two (possibly signed) measures $\mu$ and $\nu$, the bounded-Lipschitz distance between them is given by
$
    \BL(\mu, \nu) \coloneqq \sup_{f\in \mcal F_{\BL}} \abs*{\int f \dd\mu - \int f \dd\nu}\,,
$
where the supremum is taken over all bounded 1-Lipschitz functions
$\mcal F_{\BL} \coloneqq \set{f: \|f\|_{\infty}\leq 1, \abs{f(x) - f(y)}\leq \rho(x, y)}$.
By Kantorovich duality~\citep{villani2008optimal}, $\BL(\mu, \nu)= W_1(\mu,\nu)$ when $\mu$ and $\nu$ are probability measures on our diameter-one domain. The BL-distance will be useful to us since it is defined for signed measures, which naturally occur in the noisy histogram obtained from adding Gaussian noise to ensure DP.

For a dataset $D$ and
current synthetic dataset $S^{(t)}$,
let $\mu = \mu_D$ and $\mu_S^{(t)}$ to denote the empirical distribution over $D$ and $S^{(t)}$. Given variations $V^{(t)}$ generated by the $\VarAPI$, let $\mu_V^{(t)}$ be the distribution induced by the nearest neighbor histogram on $V^{(t)}$. That is, $\mu_V^{(t)}$ is supported on $V^{(t)}$ and the probability mass of a point $y$ is proportional to the number of point's in $D$ for which $y$ is that points nearest neighbor in $V^{(t)}$.
Let $\tilde \mu_V^{(t)}$ be the signed measure induced by the \emph{noisy} nearest neighbor histogram and $\bar \mu_V^{(t)}$ be the bounded-Lipschitz projection of $\tilde \mu_V^{(t)}$ onto the space of probability measures.

By repeatedly applying the triangle inequality, \citet{gonzalez2025private} bounded the $W_1$-distance between the private and synthetic datasets $W_1(\mu_D, \mu_S^{(t+1)})$ by
\begin{align*}
    &\overbrace{\E[W_1(\mu_D, \mu_V^{(t+1)})]}^{A} + 2\cdot \overbrace{\E[\BL(\mu_V^{(t+1)}, \tilde \mu_V^{(t+1)})]}^{B} + \overbrace{\E[W_1(\mu_S^{(t+1)}, \bar \mu_V^{(t+1)})]}^{C}\,.
\end{align*}
Here, the factor of 2 is due to the fact that $\BL(\tilde \mu_V, \bar\mu_V)\le \BL(\tilde\mu_V, \mu_V)$.
While \citet{gonzalez2025private} focused on bounding each term using the ambient dimension, we provide a more refined analysis using the different notions of intrinsic dimension.

Recall that the smaller  the effective dimension of the $\VarAPI$ (\Cref{asmp:reachable}),  $\lambda$, the higher ``fraction of mass'' the $\VarAPI$ puts in the direction of the data. This means that we can expect the contraction rate at each iteration (term A) to scale with $\lambda$. Indeed, we can bound the contraction rate of each iteration as a function of $\lambda$: $A\leq (1-\frac{c}{200\lambda}) W_1(\mu_D, \mu_S^{(t)})+\alpha$.

Next, using a chaining argument, it can be shown that $B\leq\tilde O\inparen{\sigma m_V^{1-\frac1{k}}}$.
Here $\sigma$ is the magnitude of Gaussian noise added to the histogram, $m_V\geq\card{V}$ is a public upper bound on the number of variations, and $k$ is the doubling dimension of the domain (\Cref{asmp:doubling-space}).
Roughly speaking, $k$ bounds the number of ``directions'' that adding Gaussian noise can warp the nearest neighbors histogram distribution $\mu_V^{(t+1)}$.

Finally, the error introduced by approximating the projected distribution via subsampling can be bounded by $C\leq\tilde O(m^{-1/k})$ using a similar chaining argument.
Choosing $m$ appropriately to balance the terms yields the desired outcome.

\section{Candidate-Selection via Clustering}\label{sec:selection-clustering}

\citet{lin2024images} originally studied PE with subsampling, where the next iteration of synthetic data points is obtained from the noisy histogram by subsampling $m$ points with replacement.
Then, \citet{xie2024text} studied PE with ranking for text data, where the $m$ points on the noisy histogram with the highest count are selected.
This section introduces \GAPE, a new variation of PE.
 
\citet{lin2024images} analyzed the case where the dataset consists of $\kappa$ points with large multiplicity
and showed the private evolution with subsampling provably converges in this setting.
However,
as we show in a simple example,
both the ranking and subsampling selection rules can perform arbitrarily worse than \GAPE, even in the slightly more general setting of well-separated clusters. Both ranking and subsampling fail to consider the relative location of histogram points and are based only on the noisy counts.
This motivates us to design a ``geometry-aware" selection mechanism which uses a clustering-based algorithm to select a representative set of candidates taking geometry into account.

\begin{example}[Ranking \& Subsampling Failure]\label{ex:ranking-subsampling-fails}
     Consider a dataset of two well-separated clusters where one cluster contains 70\% of the population. We aim to generate 3 synthetic points (ideally 2 in the majority, 1 in the minority). Even starting with this exact distribution and generating 2 variants per point, vote dilution can eliminate the minority. %
     If each cluster's votes split evenly, noiseless ranking picks only majority points; subsampling selects only majority variants at some iteration with high probability. Once the minority cluster is dropped in an iteration, it may never be recovered if variations cannot move between clusters. %
\end{example}

\myparagraph{Setting.}
In this setting,
the dataset consists of $\kappa$-clusters such that the intracluster radius is small,
but the intercluster (separation) between clusters is large.
Specifically, we assume the following assumptions hold.
\begin{assumption}[$(r, R)$-Well-Clustered]\label{asmp:separated-clusters}
    The data set consists of $\kappa$ clusters of diameter $r$ 
    such that the minimum distance between two points in different clusters is at least some given distance $R>0$ (well-separated).
    Note that we assume $R$ and $r$ are known, but not $\kappa$.
\end{assumption}

\begin{assumption}[Warm-Start]\label{asmp:warm-start}
    For each cluster, there is an initial synthetic data point $s\in S$ within distance $\nfrac{R}8$ of it.
\end{assumption}
In practice, PE is usually initialized using some prior knowledge of the dataset (\eg{}, \Cref{apx:experimental-details}).
A missed cluster may never be recovered (\Cref{ex:ranking-subsampling-fails}).

We work under the general low effective rank $\VarAPI$ model from \Cref{sec:low-dim-var-api} (\Cref{def:reachable}) but do not assume low doubling dimension (\Cref{asmp:doubling-space}).

We are now ready to state the main result of this section.
\begin{restatable}[Cluster Discovery]{theorem}{thmPrivateEvolutionClustering}\label{thm:private-evolution-clustering}
    Let $D\sset \R^d$ be a dataset of size $n$ that is $(r, R)$-well-clustered (\Cref{asmp:separated-clusters})
    and suppose that, at each iteration with synthetic dataset $S$,
    \Cref{asmp:reachable} holds with parameters $\lambda,c$
    for every $x\in D\setminus S$ and some
    $y_x\in\argmin_{y\in S}\rho(x,y)$.
    Further assume that $R\geq 400r\lambda/c$, that the initial dataset $S\sset \R^d$ satisfies the warm-start condition (\Cref{asmp:warm-start})
    and that every cluster $C_i$ satisfies $\card{C_i}>8n/m+3m_V\tau$, where $\tau=\sigma\sqrt{2\log(\nicefrac{6Tm_V}{\beta})}$ and $m$ ($m_V=\tilde O(m)$) upper bounds the number of synthetic (variation) points per iteration.
    Then after $T=\lceil(50\lambda/c)\log(\nfrac Rr)\rceil$ iterations, GAPE (\Cref{alg:private-evolution-clustering}) with noise parameter $\sigma$ outputs a synthetic dataset $S^{(T)}$ such that, with probability $1-\beta$,
    \[
        \rho(D, S^{(T)}) = \sup_{x\in D} \inf_{s\in S^{(T)}} \rho(x, s) \leq \frac{100r \lambda}{c} = O(r\lambda)\,.
    \]
    For $\eps,\delta\in(0,1)$ and
    $
        \sigma = \frac{2\sqrt{T \log(\nicefrac{1.25}{\delta})}}{\eps}\,,
    $
    the algorithm is $(\eps, \delta)$-DP and succeeds with cluster sizes
    \[
        \card{C_i} = \tilde \Omega\inparen{\frac{m_V\sqrt\lambda}{\eps}+\frac nm}\,.
    \]
\end{restatable}
Recall that $m_V=\tilde O(m)$.
Hence, for fixed $\eps$, by choosing $m=\tilde\Theta(n^{1/2}/\lambda^{1/4})$, we see that \Cref{thm:private-evolution-clustering} only requires a minimum cluster size of roughly $\tilde\Omega(n^{\frac12} \lambda^{\frac14})$.
Thus, this theorem still holds even if the  cluster sizes are very unbalanced.

\myparagraph{GAPE}
We present our Geometry-Aware PE variant, \GAPE\, in \Cref{alg:private-evolution-clustering} and the specific local search cost in \Cref{def:local-search-cost}. %
The main change from Algorithm~\ref{alg:private-evolution-worst-case} is that we use a selection mechanism that approximately minimizes a geometry aware cost function. This cost function is similar to k-median, where the goal is to choose a synthetic dataset whose empirical distribution is close in $W_1$ distance to the distribution induced by the nearest neighbor histogram. 
We adapt the k-median cost slightly to handle the negative weights that arise in the \emph{noisy} histogram.

Define a truncated distance function
    $c(x, y)\coloneqq \min\set{\rho(x, y), \nicefrac{R}3}$.

\begin{definition}[Local-Search Objective]\label{def:local-search-cost}
Given a dataset $D\sset \R^d$ with signed weight vector $\tilde H = \{w_d\}_{d\in D}$, and points $z_\ell\in \R^d$ for $\ell\in [m]$, define the following minimum cost directed flow problem with respect to this truncated distance
    \begin{align*}
        \LSCost(D, \tilde H, z) \coloneqq &
        \min_{f} \sum_{x\in D^+} \sum_{y\in D^-} c(x, y) f(x, y) + \sum_{\ell=1}^{m} \sum_{x\in D^+} c(x, z_\ell) f(x, z_{\ell}) \\
        \sum_{y\in D^-} f(x, y) + \sum_\ell f(x, z_{\ell}) &= w_x \qquad \forall x\in D^+\\
        \sum_{x\in D^+} f(x, y) &= -w_y \qquad \forall y\in D^- \\
        f &\geq 0
    \end{align*}
    where $D^+$ denotes candidate points with positive weight $w_x$, and $D^-$ denotes candidate points with negative weight $w_y$.
    For fewer than $m$ selected points, the sums over $\ell$ range over those points.
\end{definition}

We wish to select at most $m=m_S$ candidate points as the synthetic dataset in the next iteration. Our goal is to select a set of points $S^{(t)}\subseteq\tilde H_\tau^{(t)}$ that approximately minimizes $\LSCost(V^{(t)}, \tilde H^{(t)}, \cdot)$, where $\tilde H^{(t)}_{\tau}$ is the set of candidates whose noisy count exceeds $\tau$. Thresholding allows us to ensure that we are unlikely to select a data point that is not close to any of the private data.
To do this efficiently we use local search by repeatedly  swapping selected and unselected candidate points 
that decrease the local search cost (\Cref{def:local-search-cost}).
See \Cref{alg:local-search} for exact pseudocode.

\SetKwFunction{LocalSearch}{LocalSearch}
\SetKwFunction{gape}{GAPE}
\SetKwData{solution}{solution}
\begin{algorithm}[htb]
    \caption{Geometrically Aware Private Evolution\label{alg:private-evolution-clustering}}
    \Fn{\gape{$D$, $T$, $\eps$, $\delta$, $m$, $R$, $r$, $\lambda$, $c$, $\beta$}}{
        \BlankLine
        $\sigma \gets \frac{2\sqrt{T \log(\nfrac{1.25}\delta)}}{\eps}$\;
        \BlankLine
$\ell^\star\gets \lceil\log_2(200\lambda/(cr^2))\rceil$\;
        \Comment{set $2^{-\ell^\star}$ to be the smallest $\VarAPI$ scale}
        $\VarAPI(\cdot) \gets \set{\cdot}\cup \bigcup_{\ell=\lceil-\log_2(R^2)\rceil}^{\ell^\star} \VarAPI(\cdot, 2^{-\ell})^{\otimes \lceil4\log(3Tn/\beta)\rceil}$\;
        \Comment{Generate independent variations at each scale}
        $m_V\gets m\inparen{1+\inparen{\ell^\star-\lceil-\log_2(R^2)\rceil+1}\lceil4\log(3Tn/\beta)\rceil}$\;
        $\tau\gets \sigma\sqrt{2\log(\nicefrac{6Tm_V}\beta)}$\;
        \Comment{threshold value for filtering candidates}
        $S^{(0)} \gets \RanAPI(m)$\;
        \For{$t=1, \dots, T$}{
            $V^{(t)} \gets \bigcup_{y\in S^{(t-1)}} \VarAPI(y)$\;
            $H^{(t)} \gets \NNhist(D, V^{(t)}, \rho)$\;
            $\tilde H^{(t)}\gets H^{(t)} + \mcal N(0, \sigma^2 I_{\card{V^{(t)}}})$\;
            $\tilde H_\tau^{(t)} \gets \set{v\in V^{(t)}: \tilde H^{(t)}(v) > \tau}$\;
            \Comment{filter candidates for local search}
            $S^{(t)}\gets$ \LocalSearch{$V^{(t)}$, $\tilde H^{(t)}$, $\tilde H_\tau^{(t)}$, $m$}\;
            \Comment{local search (\Cref{alg:local-search})}
        }
        \Return $S^{(T)}$\;
	}

\end{algorithm}

\begin{algorithm}[htb]
    \caption{Local Search\label{alg:local-search}}
    \Fn{\LocalSearch{$V$, $\tilde H$, $\tilde H_\tau$, $m$}}{
        \lIf{$\tilde H_\tau=\varnothing$ or $\sum_{v\in V}\tilde H(v)<0$}{\Return $\{v\}$ for any $v\in V$}
        $\solution \gets$ any $\min\{m,\card{\tilde H_\tau}\}$ elements of $\tilde H_\tau$\;
        \While{true}{
            \For{$z\in \solution$ and $x\in \tilde H_\tau\setminus\solution$}{
                \If{$\LSCost(V, \tilde H, \solution\cup\set{x}\setminus\set{z}) < \LSCost(V, \tilde H, \solution)$} {
                    $\solution \gets \solution\cup\set{x}\setminus\set{z}$ \Comment{cost-decreasing swap (\Cref{def:local-search-cost})}
                    \textbf{restart the while loop}\;
                }
            }
            \Return \solution\;
        }
	}

\end{algorithm}

\myparagraph{Analysis.}
The proof of \Cref{thm:private-evolution-clustering} is detailed in \Cref{apx:private-evolution-clustering}, and we describe the proof idea below.

For a synthetic data point $s\in S^{(t)}$ within distance $\nicefrac{R}8$ of some cluster $C$, it can be shown that the generated variations $V^{(t+1)}$ contain, with high probability, a point which makes progress towards $C$.
Moreover, the support of the noiseless histogram $\supp H$, that is, the subset of $V^{(t+1)}$ with at least 1 vote, only contains points that have progressed towards some cluster.
Thus, as long as the selection mechanism chooses a point from $\supp H$ within distance $\nicefrac{R}8$ of each cluster, \GAPE\ will consistently make progress towards all clusters down to distance $O(r\lambda)$.

By restricting the local search to candidates with sufficiently high noisy votes, we can ensure that such candidates are contained in $\supp H$.
In addition, if a cluster $C$ is sufficiently large, there will be at least one point with a sufficiently high noisy vote to survive the thresholding.
It remains only to ensure that the local search terminates with a point within distance $\nicefrac{R}8$ of each cluster. We proceed to show that local search must place a point near every cluster, since otherwise swapping the center with the least incoming flow for a candidate near an uncovered cluster would decrease the cost.

\subsection{Experiments}\label{sec:experiments}
We report the various embedding distribution distances between real and synthetic data on various datasets. These experiments demonstrate that the improvements of \GAPE\ translate to real datasets.

\textbf{Datasets.} 
Following previous work~\cite{xie2024text}, we use PubMed\footnote{\url{https://www.ncbi.nlm.nih.gov}}, a dataset of $\approx75000$ abstracts of medical papers crawled by \citet{yu2023training}.
We also use Reddit\footnote{\url{https://huggingface.co/datasets/roskoN/dstc8-reddit-corpus}}, a larger dataset of $\approx5$ million Reddit dialogues curated by \citet{lee2019multi}. 

\textbf{Generators \& Semantic Embeddings.} We use
GPT-2~\cite{radford2019language} for both $\RanAPI$ and $\VarAPI$. To provide a fair comparison, for PubMed, we keep the instructional prompts and LLM temperature (1.0) the same as in~\cite{xie2024text}. 
Modified instructional prompts for Reddit can be found in the appendix. 
We also increase the temperature to 1.7.
For the embedding model, we use Sentence Transformers~\citep{reimers2019sentence} (\texttt{sentence-t5-base}, 768-dimensions). 

\textbf{Baselines.}
We compare against private evolution with subsampling~\cite{lin2024images} (\subsamplePE), and ranking~\cite{xie2024text} (\augPE).
For a fair comparison, we use the original implementation\footnote{\url{https://github.com/microsoft/DPSDA}} of private evolution for baselines and implement our new selection mechanism.

\ifnum\tpdp=0
\textbf{Setup.}
Experiments were completed using 8 Nvidia H100 GPUs (80 GB memory each). PubMed experiments ran for 3-4 hours each and Reddit experiments ran for between 3 and 8 hours each.
See \Cref{apx:experimental-details} for more setup details for our experiments.
\fi

\textbf{Results.}
Figures~\ref{fig:pubmed} and~\ref{fig:reddit} compare \GAPE\ with $\augPE$ and $\subsamplePE$ on the PubMed and Reddit datasets. Our results highlight that, as expected, $\subsamplePE$ has the worst recall. \GAPE\ is competitive with $\augPE$ in both W1-distance\footnote{For efficiency, the Sinkhorn W1 was computed instead of the true W1-distance.} and recall. For large $\epsilon$, \GAPE\ has higher recall than $\augPE$. The quality of synthetic data can be measured in a variety of ways.
We include a broader set of comparison metrics in Appendix~\ref{furtherresults}.

\begin{figure}[tbp]
    \centering
    \begin{subfigure}{0.49\linewidth}
        \centering
        \includegraphics[width=\linewidth]{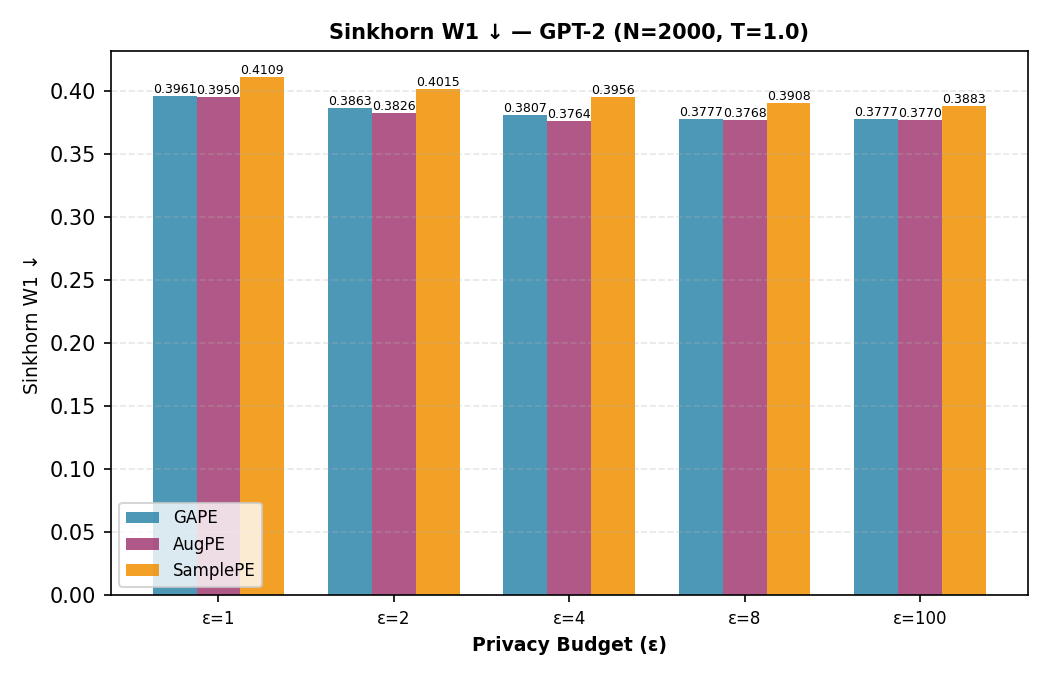}
    \end{subfigure}
    \hfill
    \begin{subfigure}{0.49\linewidth}
        \centering
        \includegraphics[width=\linewidth]{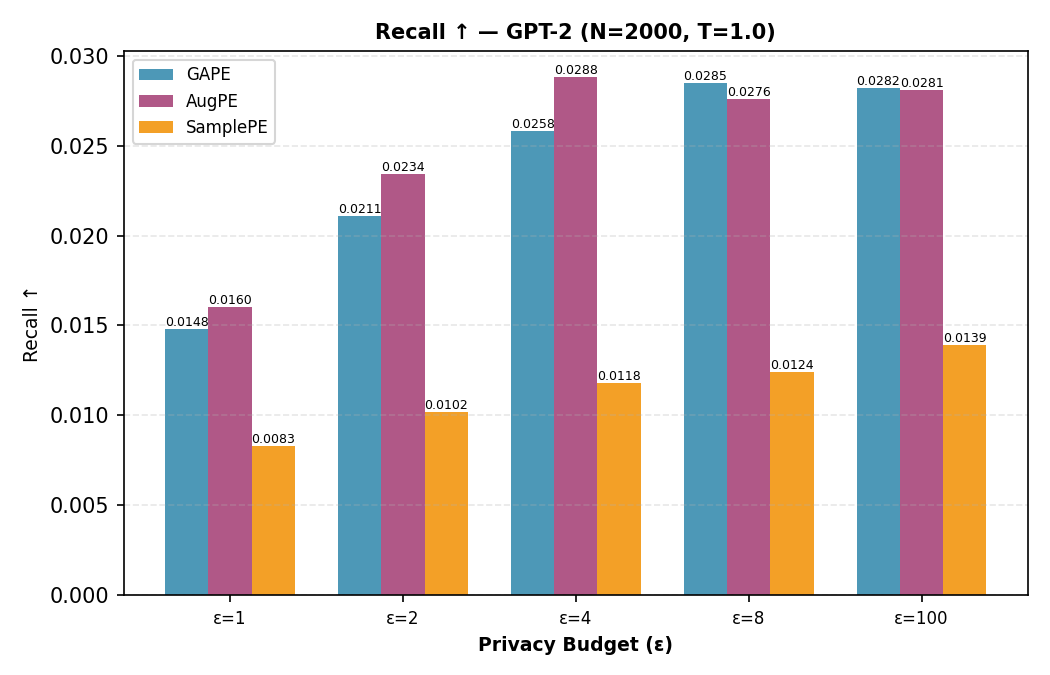}
    \end{subfigure}
    \caption{PubMed ($m=2000$): Sinkhorn $W_1$ distance (left) and recall (right).}
    \label{fig:pubmed}
\end{figure}

\begin{figure}[tbp]
    \centering
    \begin{subfigure}{0.49\linewidth}
        \centering
        \includegraphics[width=\linewidth]{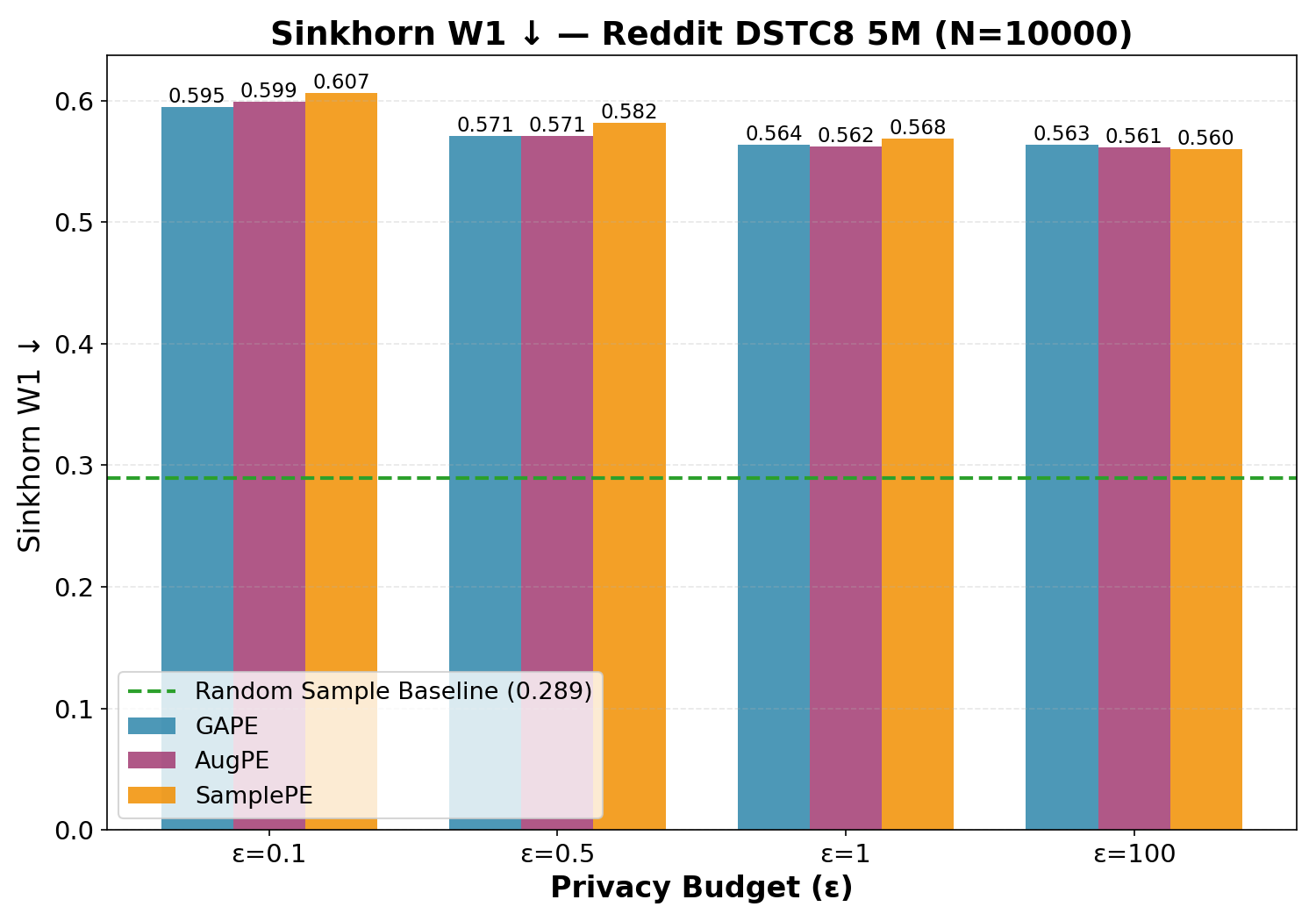}
    \end{subfigure}
    \hfill
    \begin{subfigure}{0.49\linewidth}
        \centering
        \includegraphics[width=\linewidth]{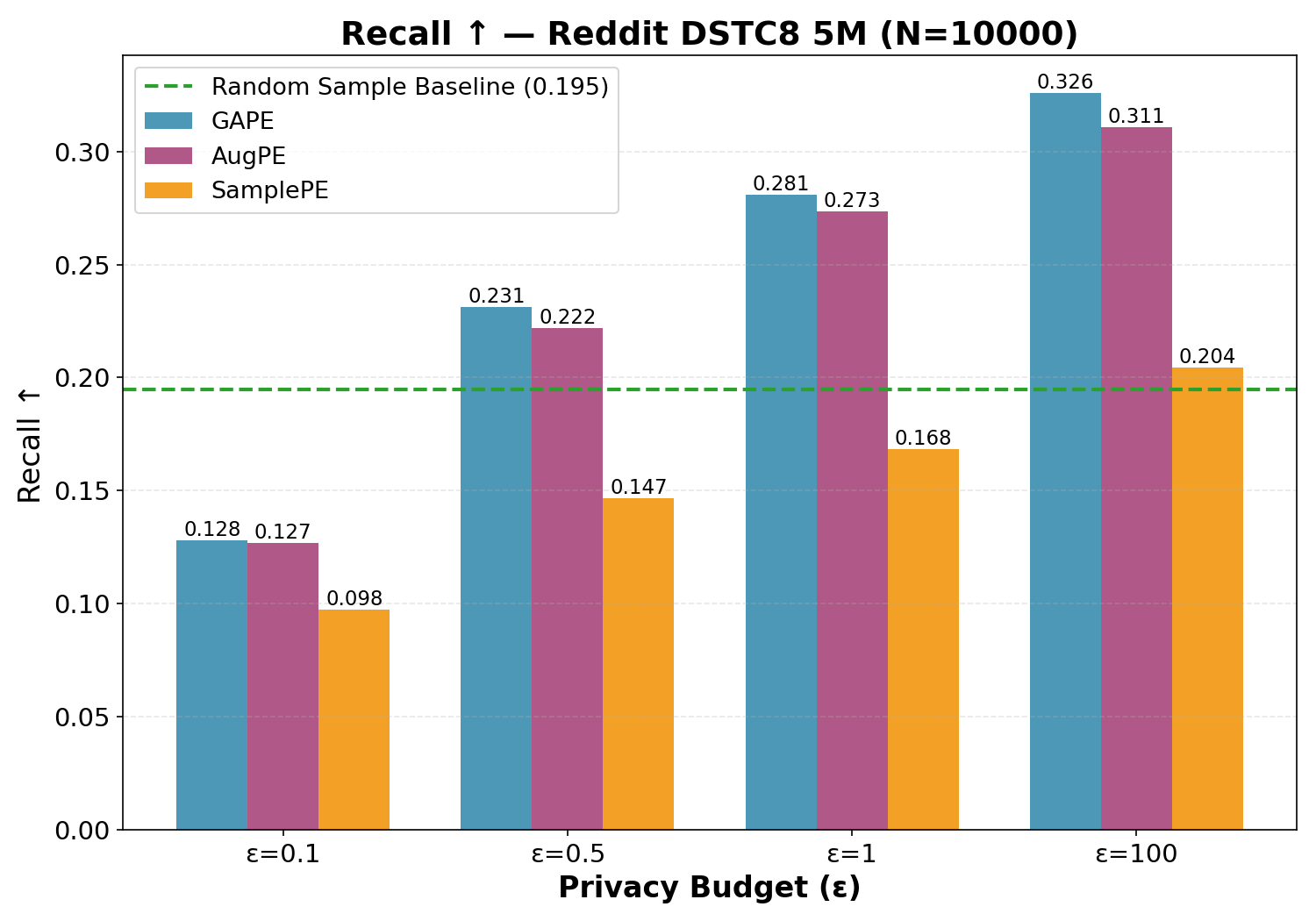}
    \end{subfigure}
    \caption{Reddit ($m=10{,}000$): Sinkhorn $W_1$ distance (left) and recall (right).}
    \label{fig:reddit}
\end{figure}

\printbibliography

\medskip

\appendix
\section{Differential Privacy Preliminaries}\label{DPprelims}
Differential privacy is a stability notion for randomized algorithms. Intuitively, it protects user's data by ensuring that the output of the algorithm does not depend too strongly on any single individual's data. 
As in \citet{xie2024text}, we will focus on approximate differential privacy. Two datasets of the same size are said to be neighboring if they differ in the data of a single individual.

\begin{definition}[$(\eps, \delta)$-Differential Privacy]\label{def:DP}
    An algorithm $\mcal A$ that takes as input an $n$-element dataset $D$ is said to be $(\eps, \delta)$-DP if for all neighboring datasets $D\sim D'$, and all events of the output space $R\sset \range(\mcal A)$,
    \[\Pr[\mcal A(D)\in R]\leq e^{\eps}\cdot \Pr[\mcal A(D')\in R] + \delta.\]
\end{definition}

Adding carefully calibrated Gaussian noise to the output of a function is a common method for achieving differential privacy.
As in other works on PE, the Gaussian mechanism is a subroutine in our implementation of PE.\footnote{As in ~\cite{xie2024text} we will use the analytic Gaussian mechanism~\cite{DBLP:conf/icml/BalleW18} in our experiments. It supports all $\eps>0$ and calibrates the noise directly to the desired privacy parameters.}

\begin{proposition}[{\citet[Gaussian Mechanism]{dwork2014algorithmic}}]\label{prop:gaussian-mechanism}
    Let $f_1, \dots, f_k$ be real-valued queries and define $\Delta_2$ to be the $\ell_2$ sensitivity over neighboring datasets: \[\Delta_2 = \max_{D,D'}\|(f_i(D))_{i\in[k]}-(f_i(D'))_{i\in[k]}\|_2,\] where the max is taken over all pairs of neighboring datasets. 
    If $\eps,\delta\in(0,1)$, and $\sigma\geq \frac{\Delta_2\sqrt{2\ln(\nfrac{1.25}\delta)}}{\eps}$, then the Gaussian mechanism $(f_i(D))_{i\in[k]}+\mcal N(0,\sigma^2 I_k)$ is $(\eps, \delta)$-DP.
\end{proposition}

A key property of differentially private algorithms is that the adaptive composition of a series of differentially private algorithms is differentially private. 
There are several ways to compute the privacy budget of the composed mechanism. 
As in other works on PE, the adaptive composition theorem of Gaussian mechanisms~\cite{dong2022gaussian} is used to compute the privacy budget of the composed mechanism in this work.
For $T$ adaptive Gaussian histogram releases, the per-round $\ell_2$ sensitivity is at most $\sqrt2$ for counts and $\sqrt2/n$ for normalized histograms, so by the composition theorem the $T$ releases satisfy the same privacy bound as a Gaussian mechanism with $\ell_2$ sensitivity $\sqrt{2T}$ (resp.\ $\sqrt{2T}/n$). Thus, for $\eps,\delta\in(0,1)$, the noise scales $\sigma=2\sqrt{T\log(\nfrac{1.25}\delta)}/\eps$ (\Cref{alg:private-evolution-template,alg:private-evolution-clustering}) and $\sigma=2\sqrt{T\log(\nfrac{1.25}\delta)}/(n\eps)$ (\Cref{alg:private-evolution-worst-case}) give $(\eps,\delta)$-DP by \Cref{prop:gaussian-mechanism}. The remaining operations are post-processing.

\section{Related Works}\label{relatedworks}

\myparagraph{DP Synthetic Data.}
The problem of private synthetic data generation has a long history, with early theoretical results showing that there are unlikely to be efficient private algorithms to generate synthetic data that accurately reflect simple statistics of the real data~\cite{DworkNRRV09,UllmanV20}. A long line of work showed oracle-efficient algorithms for such synthetic data generation tasks, under a variety of settings~\cite{BlumLR08,RothR10,HardtR10, HardtT10, GuptaRU12, HardtLM12,NikolovTZ13,GaboardiAHRW14,DworkNT14, MckennaSM19, MckennaPSM21, AydoreBKMRS21,MMSM22}. \citet{HeVZ23} phrased the problem as Wasserstein learning, and \citet{feldman2024instance} showed instance-optimality results for this formulation. ~\citet{LiuVSUW21, LiuVW21} first showed empirically that auxiliary ``public'' data can make synthetic data generation easier.
\ifnum\tpdp=0
On the practical side, there has been a flurry of activity bringing various tools to the task of private synthetic data generation~\citep{zhou2025private,Li2024privimage,hou2024PrEText,ghalebikesabi2023diffusion,harder2023pretrained,yue2023simple,kurakin2023harnessing,torfi2022medical,mattern2022secure,rosenblatt2020differentially,torkzadehmahani2019DPCGAN,abay2018privacy}. Some of these tools~\cite{kurakin-blogpost} have been used at large scale in practice to generate training data.
\fi
We refer the reader to \citep{hu2024SoK,chen2024unified} and references therein for a survey of recent developments.

\myparagraph{Private Evolution.} 
\ifnum\tpdp=1
The PE framework begins with a random dataset obtained
from a RandomAPI. It then iteratively improves the synthetic dataset, moving it closer to the private data, over T
rounds. At each iteration, it generates variations using the
VariationAPI%
Each true data point votes for the variation
most like it (its nearest variation in the embedding space) to
generate the nearest neighbor histogram. Gaussian noise is
added to achieve differential privacy. Finally, a Selection
mechanism is applied to the noisy histogram to choose the
next iteration of synthetic data.
\fi
\citet{lin2024images} designed the first iteration of private evolution for images and initiated the theoretical study of private evolution as a distribution learning algorithm under the Wasserstein metric, albeit under strong data assumptions and a data-oblivious Gaussian $\VarAPI$.
Their selection algorithm is based on subsampling.

\citet{xie2024text} adapted the private evolution framework for text and proposed a new ranking-based selection algorithm, which selects the candidates with the most noisy votes.
\ifnum\tpdp=1
\citet{gonzalez2025private} designed a variant of private evolution that is more amenable to theoretical analysis and provided worst-case $W_1$-distance guarantees under the same conditions about $\VarAPI$.
\else
\citet{gonzalez2025private} designed a variant of private evolution that is more amenable to theoretical analysis and provided worst-case $W_1$-distance guarantees under the same conditions about $\VarAPI$ as \citet{lin2024images}, but with much weaker data assumptions.
The practical success of private evolution has led to a flurry of research ~\citep{wang2025struct,gong25dpimagebench,wang2025synthesize,gonzalez2025private,zhang2025pcevolve,hou2025popri,swanberg2025api,lin2025differentially,Zou2025contrastive,hou2024PrEText,lin2024images,xie2024text}, and active industry interest~\citep{Apple2025,Microsoft2024}.
\fi

\section{Notation}\label{notation}
Throughout this work, we use the following notation.
\begin{itemize}[noitemsep,topsep=0em]
    \item $\Omega\sset\R^d$ is the data and generator domain.
    \item $n$ is the number of sensitive data points.
    \item $D$ denotes the sensitive dataset.
    \item $S$ denotes the synthetic dataset.
    \item $V\coloneqq \VarAPI(S)$ denotes the set of all variations generated from $S$.
    \item $m = m_S$ is the number of desired synthetic data points.
    \item $m_V$ is the public upper bound on the number of points generated by the $\VarAPI$ per iteration of PE.
    \item $d$ is the ambient (embedding) dimension.
    \item $\rho(x,y)=\norm{x-y}_2$ is the Euclidean distance.
    \item For a nonempty finite set $A$,
    $\rho(x,A)\coloneqq\min_{a\in A}\rho(x,a)$.
    \item $T$ is the number of iterations of PE.
    \item $\eps$ is the privacy loss parameter.
    \item $\delta$ is the additive privacy parameter.
    \item $\beta$ is the failure probability.
    \item $\sigma$ is the standard deviation of the
    Gaussian noise added to the nearest neighbor histogram.
    \item $H^{(t)}$ is the nearest neighbor histogram on $V^{(t)}$.
    \item $\tilde H^{(t)}$ is the noisy version of $H^{(t)}$.
    \item $\tilde H_\tau^{(t)}$ is the set of variations
    whose noisy count exceeds $\tau$.
    \item $\tau$ is the threshold for noisy histogram counts.
    \item $\RanAPI$ generates unconditional samples.
    \item $\VarAPI(y,s^2)$ generates variations of $y$ at scale $s^2$.
    \item $\lambda$ is the effective rank parameter.
    \item $c$ is the anticoncentration constant.
    \item $k$ is the doubling dimension parameter.
    \item $\alpha$ is the additive error of the per-iteration
    contraction in $W_1$-distance.
    \item $\kappa$ is the number of clusters.
    \item $C_i$ denotes the $i$th cluster.
    \item $r$ is the cluster diameter parameter.
    \item $R$ is the cluster separation parameter.
    \item $W_1$ denotes the 1-Wasserstein distance.
    \item $\BL$ denotes the bounded-Lipschitz distance.
    \item $\mu_D$ denotes the empirical distribution of the sensitive dataset.
    \item $\mu_S^{(t)}$ is the empirical distribution of the synthetic dataset after iteration $t$.
    \item $\mu_V$ is the distribution over the nearest neighbor histogram where each variation is assigned probability mass proportional to its vote count.
    \item $\tilde\mu_V$ is the signed measure obtained by adding Gaussian noise $\mcal N(0, \sigma^2 I_{\card{V}})$ to $\mu_V$.
    \item $\bar \mu_{V}$ is obtained by projecting $\tilde \mu_V$ onto probability measures in the bounded-Lipschitz distance.
\end{itemize}

\section{Intrinsic Dimensionality}\label{apx:intrinsic-dimensionality}
The usefulness of our work relies on the underlying assumption that for real datasets the effective dimension is significantly lower than the ambient dimension. In this section, we show that this is indeed the case for the PubMed dataset used in our experiments. Following the work of \citet{lin2024images}, we estimate the intrinsic dimension of text embeddings with the following process:
\begin{enumerate}[1),noitemsep,topsep=0em]
    \item Sample a PubMed text $x$. Compute its embedding vector $g$ with \texttt{sentence-t5-base}.
    \item We use GPT2 as the $\VarAPI$ to generate 3000 text variations of $x$ and compute their respective embeddings $g_1, \dots, g_{3000}$.
    \item We construct a matrix $M = [g_i - g]_{i\in [3000]}\in \R^{3000\times 768}$.
    \item We compute the singular values of $M: \sigma_1\geq \sigma_2\geq \dots\geq \sigma_{768}$
    \item We compute the minimum number of singular values $k$ so that the explained variance ratio satisfies
    \[
        \frac{\sum_{i=1}^k \sigma_i^2}{\sum_{i=1}^{768} \sigma_i^2} \geq 0.8\,.
    \]
\end{enumerate}

Intuitively, this captures the number of dimensions needed to accurately reconstruct the embedding differences $M$ with small error.
    We use it as an estimated intrinsic dimension of the text variations.

We repeat the above at 3 points $x$ from PubMed and at various temperature parameters for GPT2 generation.
The minimum, median, and maximum are plotted in \Cref{fig:variation-api-intrinsic-dim}.
As one might expect, the intrinsic dimension grows with the temperature but, even for higher temperatures, is an order of magnitude (\eg{}, $20 \ll 768$) less than the ambient embedding dimension.

\begin{figure}[H]
    \centering
    \includegraphics[width=0.5\linewidth]{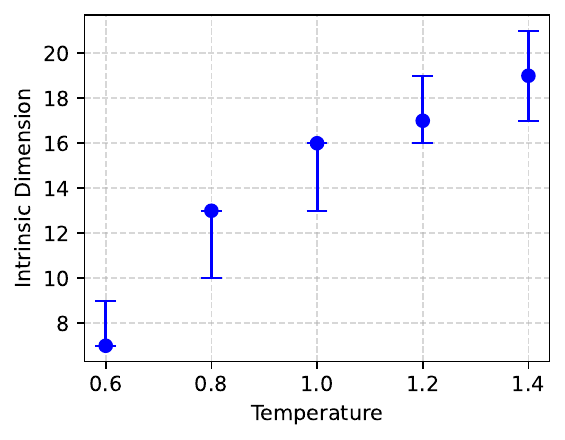}
    \caption{The estimated intrinsic dimension of GPT-2 as the $\VarAPI$ at different temperatures on random points in the PubMed dataset.}
    \label{fig:variation-api-intrinsic-dim}
\end{figure}

Next, we fix a single sample $x$ from PubMed and plot the singular value decay rate.
This is presented in \Cref{fig:variation-api-svd-decay} with the $x$-axis on a log-scale in order to distinguish the different temperatures.

\begin{figure}[H]
    \centering
    \includegraphics[width=0.5\linewidth]{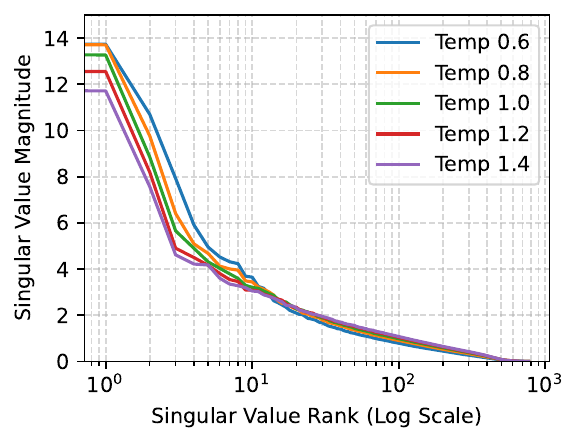}
    \caption{The singular value decay for the matrix of 3000 embedding difference vectors from a random entry in the PubMed dataset. Variations are generated by GPT-2 at various specified temperatures.}
    \label{fig:variation-api-svd-decay}
\end{figure}

To highlight that this phenomenon is not unique to GPT-2, we computed the intrinsic dimension for Mistral-7B-Instruct-v0.2 and Llama-3.2-1B. For Mistral-7B-Instruct-v0.2 (ambient dimension 768), the intrinsic dimension varied between 10 and 21 for temperature between 0.6 and 1.4. For Llama-3.2-1B (ambient dimension 768), the intrinsic dimension varied between 4 and 32 for temperature between 0.6 and 1.4. All substantially smaller than the ambient dimension of 768.

\section{Proof of \texorpdfstring{\Cref{thm:private-evolution-worst-case}}{Theorem}}\label{apx:private-evolution-worst-case}
We now prove \Cref{thm:private-evolution-worst-case},
which is restated below for convenience.
\thmPrivateEvolutionWorstCase*

\begin{example}\label{ex:reachability-2d}
    Consider a dataset consisting of a single point $x = (0, 1)^\top$, and initialize PE with a single synthetic point at the origin $y = (0, 0)^\top$.
    The generator $\VarAPI((y_1, y_2)^\top, s^2)$ is modeled by the diagonal Gaussian distribution $\mcal N\inparen{(y_1, y_2)^\top, \frac{s^2}{1+\gamma} \begin{bsmallmatrix} 1 & 0 \\ 0 & \gamma \end{bsmallmatrix}}$,
    where $\gamma\in [0, 1]$ is the amount of ``coverage'' of the generator in the second coordinate, with $\gamma = 0$ being no coverage.
    For $\gamma>0$, the effective rank (\Cref{def:reachable}) of the generator at the origin with respect to $x$ is
    \[
        \frac{\tr\inparen{\frac{s^2}{1+\gamma} \begin{bmatrix} 1 & 0 \\ 0 & \gamma \end{bmatrix}} \cdot \norm{x-y}_2^2}{(x-y)^\top \frac{s^2}{1+\gamma} \begin{bmatrix} 1 & 0 \\ 0 & \gamma \end{bmatrix} (x-y)}
        = \frac{s^2\cdot1}{s^2\gamma/(1+\gamma)}
        = \frac{1+\gamma}{\gamma}\,,
    \]
    so that $\lambda = \Theta(\nfrac1\gamma)$.
    If $\gamma = 0$, the generator only ever generates variations along the first coordinate, \ie{}, all variations have $0$ in their second coordinate.
    Hence no finite effective rank suffices, and PE cannot converge to $x$.
    This illustrates that without a compatibility condition on the generator, PE may not converge.
    On the other hand, for $\gamma\in (0, 1]$, the generator has variance $s^2\gamma/(1+\gamma)$ in the second coordinate. Intuitively, smaller variations in the second coordinate will require more steps to converge to $x$.
    This is reflected in the contraction bound of \Cref{lem:contraction-one-scale}, through $\lambda=\Theta(\nfrac1\gamma)$.
\end{example}

\subsection{Contraction via Variations}
In this subsection, we quantify the reduction in $W_1$-distance between the synthetic and sensitive datasets from applying the variation API in a single iteration.

\begin{lemma}\label{lem:contraction-one-scale}
    Fix $x, y\in \Omega$ and suppose
    \[
        s^2\in \left[ c\norm{x-y}_2^2/400\lambda, c\norm{x-y}_2^2/200\lambda \right].
    \]
    For $Z\sim\VarAPI(y,s^2)$, assume $\E[Z]=y$ and $\tr\Cov[Z]=s^2$.
    Suppose $x$ is reachable (\Cref{asmp:reachable}) from $\VarAPI(y, s^2)$ with effective rank $\lambda$ (\Cref{def:reachable}).
    Then with probability $\nicefrac14$ over the draw $z\sim \VarAPI(y, s^2)$,
    \[
        \norm{z-x}_2^2
        \leq \left( 1-\frac{c}{25\lambda} \right) \norm{x-y}_2^2\,.
    \]
\end{lemma}

\begin{pf}
    Write $z=y+\xi\sim \VarAPI(y, s^2)$ for some centered random variable $\xi$.
    We have
    \begin{align*}
        \norm{y+\xi-x}_2^2
        &= \norm{y-x}_2^2 - 2\iprod{\xi, x-y} + \norm{\xi}_2^2\,.
    \end{align*}
    By reachability and the moment assumptions,
    we know that with probability $\nicefrac13$,
    \[
        \iprod{\xi, x-y}
        \geq \sqrt{\frac{c\cdot \tr\Cov[\xi]}\lambda}\cdot \norm{x-y}_2
        = \sqrt{\frac{cs^2}\lambda}\cdot \norm{x-y}_2\,.
    \]
    On the other hand,
    we know that $\E[\norm{\xi}_2^2] = \tr\Cov[\xi] = s^2$.
    Hence, by a simple Markov inequality,
    $\norm{\xi}_2^2 \leq 12s^2$ with probability $1-\nicefrac1{12}$.

    Then, by a union bound, the following holds with probability $\nicefrac14$:
    \begin{align*}
        \norm{y+\xi-x}_2^2
        &\leq \norm{x-y}_2^2 - 2 \sqrt{\frac{cs^2}\lambda}\cdot \norm{x-y}_2 + 12s^2\,.
    \end{align*}
    For $s^2\in \left[ c\norm{x-y}_2^2/400\lambda, c\norm{x-y}_2^2/200\lambda \right]$, this expression simplifies to
        \begin{align*}
        \norm{y+\xi-x}_2^2
        &\leq \norm{x-y}_2^2 - \frac{2(\nicefrac1{20} - \nicefrac6{200}) c}{\lambda} \norm{x-y}_2^2 \\
        &= \inparen{1-\frac{c}{25\lambda}}\norm{x-y}_2^2\,. \qedhere
    \end{align*}
\end{pf}

\begin{corollary}\label{cor:expected-contraction-all-scale}
    Let $S\sset \Omega$ and fix $x\in\Omega$ and $y_x\in\argmin_{y\in S}\rho(x,y)$.
    If $x\notin S$, suppose that \Cref{asmp:reachable}
    holds for $x,y_x$ with parameters $\lambda,c$.
    Let $\ell^\star\coloneqq\left\lceil\log_2\inparen{\frac{200\lambda}{c\alpha^2}}\right\rceil$ and set $2^{-\ell^\star}$ to be the smallest scale of $\VarAPI$.
    If $V\sim \VarAPI(S)$ is the union of $S$ and a single variation from each scale, then
    \begin{align*}
        &\E_{V\sim \VarAPI(S)}\left[ \min_{v\in V}\norm{x-v}_2^2 \right] \leq \max\set*{ \alpha^2, \inparen{1-\frac{c}{100\lambda}}\norm{x-y_x}_2^2 }\,.
    \end{align*}
\end{corollary}

\begin{proof}
    First we note that $y_x\in V$ so that $\min_{v\in V} \norm{x-v}_2^2\leq \norm{x-y_x}_2^2$.
    Thus if $\norm{x-y_x}_2 \leq \alpha$, there is nothing to prove.
    We proceed assuming $\norm{x-y_x}_2 > \alpha$.

    In this case, one of the scales falls into the interval $s^2\in \left[ c \norm{x-y_x}_2^2/400\lambda, c \norm{x-y_x}_2^2/200\lambda \right]$ required by \Cref{lem:contraction-one-scale}.
    The moment and reachability conditions at this scale
    follow from \Cref{def:variation-api,asmp:reachable}.
    Thus, with probability at least $\nfrac14$, the distance shrinks by $(1-\frac{c}{25\lambda})$.
    With the remaining probability, the distance remains the same.
    This concludes the proof.
\end{proof}

\begin{lemma}[Lemma E.2 in \citep{gonzalez2025private}]\label{lem:norm-to-W1-contraction}
    Fix $\gamma\in (0, 1)$.
    Let $D, S\sset \Omega$ be two datasets and $\mu_D, \mu_S$ be the empirical distributions over the respective datasets.
    Let $V = \VarAPI(S)\sset \Omega$ and $\mu_V$ be the empirical distribution induced by nearest neighborhood histogram of private evolution (\Cref{alg:private-evolution-template}).
    Suppose for $x\in D$,
    \[
        \E_{V\sim \VarAPI(S)}\left[ \min_{v\in V}\rho(x, v) \right]
        \leq \max\set*{ \alpha, (1-\gamma)\rho(x, y_x) }\,,
    \]
    where $y_x\in \argmin_{y\in S} \rho(y, x)$.
    Then
    \[
        \E_{V\sim \VarAPI(S)}\left[ W_1(\mu_V, \mu_D) \right]
        \leq (1-\gamma) W_1(\mu_S, \mu_D) + \alpha\,.
    \]
\end{lemma}

Using the fact that $\sqrt{1-t}\leq 1-\nfrac{t}2$ for $t\leq 1$,
\Cref{cor:expected-contraction-all-scale,lem:norm-to-W1-contraction} imply a $(1-\gamma)$ contraction in the $W_1$-distance for $\gamma = \Theta(\nfrac1\lambda)$ after a single private evolution iteration, up to an additive factor of $\alpha$.

\subsection{Private Histogram Error}
Next,
we bound the (dual) $W_1$-error from adding noise to obtain privacy
and the projection step.

Let $\mcal F\sset \R^\Omega$ be a real-valued function class over the domain $\Omega$.
Recall the following notions of the complexity of $\mcal F$.
\begin{definition}[Complexity of $\mcal F$]\label{def:function-complexity}
    Let $P$ be the standard Gaussian ($\mcal N(0, 1)$),
    standard Laplace ($\Lap(1)$),
    or Rademacher ($\Rad(\pm1)$) distribution.
    The $P_n$-complexity of $\mcal F$ is given by
    \[
        \sup_{z_1, \dots, z_n\in \Omega} \E_{\xi_1, \dots, \xi_n\sim_{i.i.d.} P} \left[ \sup_{f\in\mcal F}\frac1n \sum_{i\in [n]} \xi_i f(z_i) \right]\,.
    \]
\end{definition}
We are specifically interested in the following dual definition of the $W_1$ distance,
\[
    W_1(\mu, \nu) = \sup_{f\in \mcal F_{\BL}} \int_{\Omega} f(z) \mu(z) - f(z) \nu(z)
    \eqqcolon \BL(\mu, \nu)
\]
where $\mcal F_{\BL}$ is the collection of bounded $1$-Lipschitz functions over $\Omega$ defined in \Cref{sec:proofoutline}.
When $\mu = \mu_V$ is the distribution induced by the nearest neighbors histogram and $\nu = \mu_V + \mcal N(0, \sigma^2 I_{\card{V}})$ is the signed measure obtained by adding noisy to ensure privacy, the $W_1$-distance is no longer defined but the $\BL$ (bounded-Lipschitz) distance extends naturally to this setting.

\begin{lemma}[Lemma E.3 in \citep{gonzalez2025private}]\label{lem:noisy-histogram-error}
    Let $V\in \Omega^n$,
    $\mu\in \Delta_V$ be a discrete distribution supported on $V$,
    and $Z\sim \mcal N(0, \sigma^2 I)$.
    Write $\mcal G_n(\cdot)$ to denote the Gaussian complexity.
    Then
    \[
        \E_Z\insquare{\BL(\mu, \mu+Z)}
        \leq n\sigma \mcal G_n(\mcal F_{\BL})\,.
    \]
\end{lemma}
The rest of the section is dedicated to determining $\mcal G_n(\mcal F_{\BL})$

For a set $T$ with metric $\rho$,
recall that the \emph{$\beta$-covering number} $N(T, \rho; \beta)$
is the smallest cardinality of a $\beta$-cover of $T$.
It is known that we can control the function complexity via chaining arguments.
\begin{lemma}[See \eg{}, Lemma E.4 in \citep{gonzalez2025private}]\label{lem:chaining}
    Let $(\Omega, \rho)$ be a metric space
    and $\mcal F\sset \R^\Omega$ be some class of real-valued functions over $\Omega$ with $0\in\mcal F$ and $\norm{f}_\infty\leq1$ for every $f\in\mcal F$.
    There is an absolute constant $C > 0$ such that
    \[
        \mcal G_n(\mcal F)
        \leq C\cdot \inf_{\alpha\in[0,1]} \insquare{\alpha + \frac1{\sqrt n}\int_\alpha^1 \sqrt{\log N(\mcal F, \norm{\cdot}_\infty; \beta)} d\beta}\,.
    \]
\end{lemma}

Our goal is now to bound the covering number given the additional information that the domain $\Omega$ has bounded doubling dimension $\ddim(\Omega) = k$.

\begin{lemma}\label{lem:domain-covering-number}
    Let $(\Omega, \rho)$ be a metric with doubling dimension at most $k$ and diameter 1.
    Write $\mcal F_{\BL}$ to denote the class of bounded 1-Lipschitz functions over $\Omega$ defined in \Cref{sec:proofoutline}.
    Then for any $\beta\in (0, 1)$,
    \[
        \log N(\mcal F_{\BL}, \norm{\cdot}_\infty; \beta)
        \leq \inparen{\frac5\beta}^k\log\inparen{\frac8\beta}\,.
    \]
\end{lemma}

\begin{proof}
    By \citet[Lemma 4.2]{gottlieb2016adaptive},
    we know that
    \[
        N(\mcal F_{\BL}, \norm{\cdot}_\infty; \beta)
        \leq \inparen{\frac8\beta}^{N(\Omega, \rho; \nicefrac\beta2)}\,.
    \]
    Now, $\Omega$ lies in a ball of radius $1$.
    Hence, it can be covered by $2^k$ ball of radius $\nfrac12$, each of which can also be covered by $2^k$ balls of radius $\nfrac14$.
    Thus, recursively applying the definition of the doubling dimension yields that 
    \[
        N(\Omega, \rho; \nicefrac\beta2)
        \leq (2^k)^{\lceil\log_2(\nfrac2\beta)\rceil}
        \leq \inparen{\frac{5}\beta}^k\,. \qedhere
    \]
\end{proof}

Combining \Cref{lem:chaining,lem:domain-covering-number} yields the following bound on the Gaussian complexity of $\mcal F_{\BL}$.
\begin{proposition}\label{prop:gaussian-complexity}
    Let $(\Omega, \rho)$ be a metric with doubling dimension at most $k$ and diameter 1.
    Write $\mcal F_{\BL}$ to denote the class of bounded 1-Lipschitz functions over $\Omega$ defined in \Cref{sec:proofoutline}.
    Then
    \[
        \mcal G_n(\mcal F_{\BL}) \leq
        O\inparen{\frac{\log^{3/2}(8n)}{n^{1/\max(2,k)}}}\,.
    \]
\end{proposition}

\begin{proof}
    By \Cref{lem:chaining,lem:domain-covering-number},
    \[
        \mcal G_n(\mcal F_{\BL})
        \leq C\cdot\inf_{\alpha\in(0,1]}\insquare{
            \alpha+\frac{5^{k/2}}{\sqrt n}\int_\alpha^1
            \beta^{-k/2}\sqrt{\log(8/\beta)}\,d\beta}\,.
    \]
    Take $\alpha=\min\set{1,5n^{-1/\max(2,k)}}$.
    If $\alpha=1$, the bound follows from $\mcal G_n(\mcal F_{\BL})\leq1$.
    Otherwise, for $k\leq2$, use $5^{k/2}\leq5$ and
    $\beta^{-k/2}\leq\beta^{-1}$.
    For $k>2$, use
    \[
        \int_\alpha^1\beta^{-k/2}\,d\beta
        \leq\alpha^{1-k/2}\log(1/\alpha),
        \qquad
        \frac{5^{k/2}}{\sqrt n}\alpha^{1-k/2}=5n^{-1/k}\,.
    \]
    In both cases the claimed bound follows, with an absolute implicit constant.

\end{proof}

\Cref{lem:noisy-histogram-error,prop:gaussian-complexity} together yield a bound on the expected $\BL$-distance of the noiseless and noisy histogram distributions on the order of
$
    \tilde O\inparen{\sigma m_V^{1-\frac1{\max(2, k)}}}\,,
$
where $\card{V}=\card{\VarAPI(S)}\leq m_V$ and $m_V$ is a public upper bound on the number of variations generated by the variation API.
In this section we take $m_V=m(\ell^\star+2)$.

\subsection{Subsampling Error}
We complete the proof by determining the error introduced by subsampling to reduce the size of the support of $S$.
By the Kantorovich-Rubinstein duality,
This is essentially a question of convergence of the empirical measure $\mu_n$ in $W_1$-distance for a bounded distribution $\mu$.
\begin{align*}
    &\E\insquare{W_1(\mu_n, \mu)} \\
    &= \E_{X\sim \mu_n} \insquare{\sup_{f\in \mcal F_{\BL}} \abs*{\frac1n \sum_{i=1}^n f(X_i) - \E_{Y\sim \mu}[f(Y)]}}\,.
\end{align*}

We show that this quantity is in fact bounded above by the Gaussian complexity $\mcal G_n(\mcal F_{\BL})$.
\begin{lemma}\label{lem:subsampling-error}
    Let $\mu$ be a probability measure and $\mu_n$ be its $n$-sample empirical measure.
    The following holds.
    \[
        \E\insquare{W_1(\mu_n, \mu)}
        \leq \sqrt{2\pi}\cdot \mcal G_n(\mcal F_{\BL})\,.
    \]
\end{lemma}

The proof is via a standard symmetrization argument.
\begin{proof}
    We first bound the uniform convergence error using the Rademacher complexity (\Cref{def:function-complexity}).
    \begin{align*}
        &\E\insquare{W_1(\mu_n, \mu)} \\
        &= \E_{X\sim \mu_n} \insquare{\sup_{f\in \mcal F_{\BL}} \abs*{\frac1n \sum_{i=1}^n f(X_i) - \E_{Y\sim \mu}[f(Y)]}} \\
        &\leq \E_{X, X'\sim_{i.i.d.} \mu_n} \insquare{\sup_{f\in \mcal F_{\BL}} \abs*{\frac1n \sum_{i=1}^n f(X_i) - f(X_i')}} \tag{Jensen's inequality} \\
        &= \E_{\eps\sim \Rad_n(\pm1)} \E_{X, X'\sim_{i.i.d.} \mu_n} \insquare{\sup_{f\in \mcal F_{\BL}} \abs*{\frac1n \sum_{i=1}^n \eps_i[f(X_i) - f(X_i')]}} \tag{symmetry} \\
        &\leq 2\E_{\eps, X} \insquare{\sup_{f\in \mcal F_{\BL}} \abs*{\frac1n \sum_{i=1}^n \eps_if(X_i)}} \tag{triangle inequality} \\
        &\leq 2\cdot \mcal R_n(\mcal F_{\BL}).
    \end{align*}
    Now,
    by from elementary high-dimensional probability,
    we know that $\mcal R_n(\mcal F) \leq \sqrt{\pi/2}\cdot \mcal G_n(\mcal F)$ for any function class $\mcal F$~\citep{wainwright2019high,bartlett2002rademacher,tomczak1989banach}.
    This concludes the proof.
\end{proof}

\Cref{lem:subsampling-error} allows us to control the error incurred by sparsifying the variations using the Gaussian complexity of $\mcal F_{\BL}$.
To bound the latter, we can simply reuse the bound that we already developed.

\subsection{Putting it Together}
In this subsection, we combine the ingredients from previous subsections to obtain a worst-case guarantee for private evolution in $W_1$-distance, which depends on the intrinsic dimension of $\VarAPI$ rather than the ambient dimension.

We are now ready to prove \Cref{thm:private-evolution-worst-case}.
\begin{proof}
    Let $\Gamma_t \coloneqq \E[W_1(\mu_S^{(t)}, \mu)]$,
    where the expectation is taken over the randomness of the algorithm.
    If $\sigma>1$, the claimed bound follows from the diameter-one assumption. Hence suppose that $\sigma\leq1$.
    For $\gamma\coloneqq \frac{c}{200\lambda}$,
    we have
    \begin{align*}
        \Gamma_{t+1}
        &\leq \E[W_1(\mu, \mu_V^{(t+1)})] + 2\cdot \E[\BL(\tilde \mu_V, \mu_V)] + \E[W_1(\mu_S^{(t+1)}, \bar\mu_{V}^{(t+1)})] \tag{BL-projection} \\
        &\leq (1-\gamma) \Gamma_t + \alpha + 2m_V \sigma \mcal G_{m_V}(\mcal F_{\BL}) + \sqrt{2\pi} \mcal G_{m}(\mcal F_{\BL}) \tag{\Cref{cor:expected-contraction-all-scale,lem:norm-to-W1-contraction,lem:noisy-histogram-error,lem:subsampling-error}} \\
        &\leq (1-\gamma)^{t+1} \Gamma_0 + \frac1\gamma \inparen{\alpha + 2m_V \sigma \mcal G_{m_V}(\mcal F_{\BL}) + \sqrt{2\pi} \mcal G_{m}(\mcal F_{\BL})} \tag{geometric series} \\
        &\leq (1-\gamma)^{t+1} \Gamma_0 + \tilde O\inparen{\gamma^{-1}\inparen{\alpha+\sigma m^{1-\frac1{\max(2,k)}}+m^{-\frac1{\max(2,k)}}}}\,. \tag{$m_V\leq m(\ell^\star+2)=O\inparen{m\log(\nicefrac1{\gamma\alpha^2})}$}
    \end{align*}
    Roughly speaking, we would like to balance $m \sigma \mcal G_{m}(\mcal F_{\BL}) + \mcal G_{m}(\mcal F_{\BL})$.
    Setting $m = \lceil\sigma^{-1}\rceil$ gives $\sigma^{-1}\leq m\leq2\sigma^{-1}$ and yields the upper bound
    \begin{align*}
        \Gamma_t
        &\leq(1-\gamma)^t\Gamma_0+\tilde O\inparen{\gamma^{-1}\inparen{\alpha+\sigma^{\frac1{\max(2,k)}}}}\,.
    \end{align*}

    Finally,
    choosing $\alpha = \sigma^{\frac1{\max(2, k)}}$
    and $T\geq \gamma^{-1}\log\inparen{\gamma/\sigma^{\frac1{\max(2, k)}}}$ yields
    \begin{align*}
        \Gamma_T
        &\leq\tilde O\inparen{\gamma^{-1}\sigma^{\frac1{\max(2,k)}}}\,.
    \end{align*}

    Now, the normalized nearest neighbor histogram has $\ell_2$-sensitivity $\frac{\sqrt2}{n}$.
    Thus in order to ensure that the algorithm is $(\eps, \delta)$-DP,
    it suffices to set
    $
        \sigma 
        = \frac{2\sqrt{T \log(\nicefrac{1.25}{\delta})}}{n\eps}
    $
    (\Cref{prop:gaussian-mechanism}).
    Since we also require $T\geq \gamma^{-1}\log\inparen{\gamma/\sigma^{\frac1{\max(2, k)}}}$,
    it suffices to set
    \begin{align*}
        T 
        &\coloneqq \left\lceil\gamma^{-1}\log(2+n\eps)\right\rceil \\
        &\geq \frac{\log(\gamma / \sigma^{\frac1{\max(2, k)}})}{\gamma}\,. \tag{$\gamma\leq1/100$, $k\geq2$, $T\geq1$}
    \end{align*}
    
    Indeed, writing $a=2\sqrt{\log(1.25/\delta)}>0.94$, we have
    $\gamma/\sigma^{1/k}\leq\gamma(n\eps/a)^{1/k}\leq2+n\eps$:
    for $n\eps\leq1$ the middle expression is less than one,
    and for $n\eps\geq1$ it is at most $n\eps$.

    All in all,
    we conclude that
    \begin{align*}
        \Gamma_T
        &\leq\tilde O\inparen{\gamma^{-1}\sigma^{\frac1{\max(2,k)}}}\\
        &=\tilde O\inparen{\lambda^{1+\frac1{2\max(2,k)}}\inparen{\frac{\sqrt{\log(2+n\eps)\log(\nicefrac{1.25}\delta)}}{\eps n}}^{\frac1{\max(2,k)}}}\,.
    \end{align*}

\end{proof}

\section{Proof of \texorpdfstring{\Cref{thm:private-evolution-clustering}}{Theorem}}\label{apx:private-evolution-clustering}
In this section, we present the deferred proof of \Cref{thm:private-evolution-clustering}, which is restated below for convenience.
\thmPrivateEvolutionClustering*

We also detail the exact local search procedure in \Cref{alg:local-search}.

\subsection{Cluster Discovery}

We first record the generator event used below. Suppose every cluster has
a synthetic point within distance $\nfrac R8$.
For a point $x\in D$ with $\rho(x,S)>r$, choose a
closest synthetic point $y_x$ satisfying \Cref{asmp:reachable}.
One of the sampled scales lies between $c\rho(x,S)^2/(400\lambda)$ and
$c\rho(x,S)^2/(200\lambda)$.
By \Cref{def:variation-api,asmp:reachable}, the moment and reachability
conditions hold at this scale.
By \Cref{lem:contraction-one-scale}, the independent draws at that scale
fail to yield the stated contraction with probability at most
$(3/4)^{\lceil4\log(3Tn/\beta)\rceil}\leq\beta/(3Tn)$.
For $\rho(x,S)\leq r$, the retained original candidate already has distance
at most $r$. Thus, with probability at least $1-\nfrac\beta{3T}$,
\[
    \rho(x,V)\leq\max\inbrace{r,
      \inparen{1-\frac{c}{50\lambda}}\rho(x,S)}
    \qquad\text{for every }x\in D\,.
\]
Since $\lambda\geq1$ and $c\leq2$ (by Cantelli's inequality), the assumption $R\geq400r\lambda/c$
implies $R\geq200r$ and
\[
    \max\inbrace{r,\inparen{1-\frac{c}{50\lambda}}
      \inparen{\frac R8+r}}\leq\frac R8\,.
\]
On this event, every noiseless vote goes to a candidate within distance
$\nfrac R8$ of its voter. Conditional on the generated candidates, Gaussian
concentration gives $\|\tilde H-H\|_\infty\leq\tau$ with probability at
least $1-\nfrac\beta{3T}$ for
$\tau=\sigma\sqrt{2\log(\nicefrac{6Tm_V}{\beta})}$.
The next lemma is deterministic on these two events.

\begin{lemma}\label{lem:localSearch:retainClusters}
    Suppose the dataset is $(r,R)$-well-clustered with $R\geq200r$,
    and every $x\in D$ votes for a candidate within distance $\nfrac R8$.
    Assume $\card V\leq m_V$ and $\|\tilde H-H\|_\infty\leq\tau$, and
    \[
        \card{C_i}>\frac{8n}{m}+3m_V\tau
        \qquad\text{for every cluster }C_i\,.
    \]
    Then local search over thresholded candidates $\tilde H_\tau$ with
    the minimum cost flow objective selects at most $m$ candidates $S'$
    such that for every cluster $C_i$, there is some $y_i\in S'$ with
    $\rho(y_i,C_i)\leq\nicefrac R8$.

\end{lemma}

\begin{proof}
Every thresholded candidate received a noiseless vote, since
$H(v)=0$ implies $\tilde H(v)\leq\tau$.
Each candidate in $\supp H$ receives votes from only one cluster:
two voters for the same candidate are at distance at most $\nfrac R4$.
Every cluster has at most $m_V$ voted-for candidates, all within
distance $\nfrac R8$ of it. Since $\card{C_i}>2m_V\tau$, the
pigeonhole principle guarantees a candidate with more than $2\tau$
noiseless votes, and hence more than $\tau$ noisy votes.
The cluster-size assumption implies $m>8$ and $n>3m_V\tau$.
The total negative weight is at most $m_V\tau$, while
$\sum_{v\in V}\tilde H(v)\geq n-m_V\tau>0$, so the flow problem
is feasible.
If there are at most $m$ thresholded candidates,
selecting all of them gives the conclusion.

Otherwise, suppose towards a contradiction that local search selects
$m$ candidates $S'$ but none within distance $\nfrac R8$ of $C_i$.
Fix an optimal flow for $S'$. The total incoming flow to selected
candidates is $\sum_{v\in V}\tilde H(v)\leq n+m_V\tau$.
Thus some selected candidate $y_2$ receives at most
$(n+m_V\tau)/m$ flow.
Let $y_1$ be any thresholded candidate within distance $\nfrac R8$
of $C_i$, and let $y_3$ be any other selected candidate.

It remains to exhibit a re-selection and feasible flow with strictly
lower cost. We close $y_2$ by reassigning its incoming flow to $y_3$,
then open $y_1$ and redirect to it the flow sent to selected
candidates by positive-weighted candidates receiving noiseless
votes from $C_i$.
We do not change any flow into negative-weighted candidates.
Since costs are truncated at $\nfrac R3$, the cost of closing $y_2$
is at most
\[
    \frac R3\frac{n+m_V\tau}{m}\,.
\]

The candidates receiving noiseless votes from $C_i$ have total
positive weight at least $\card{C_i}-m_V\tau$.
At most $m_V\tau$ of this weight flows to negative-weighted
candidates, leaving at least $\card{C_i}-2m_V\tau$ flowing to
selected candidates. Each such source is within $\nfrac R8$ of
some voter in $C_i$, as is $y_1$. Its distance to $y_1$ is therefore
at most $\nfrac R4+r$, whereas its distance to every previously
selected candidate is at least $\nfrac{3R}4$, since such a candidate lies in
$\supp H$ and hence within $\nfrac R8$ of a voter outside $C_i$.
On the other hand, the gain from opening $y_1$ is at least
\[
    \inparen{\frac R3-\frac R4-r}
    \inparen{\card{C_i}-2m_V\tau}
    \geq\frac R{16}\inparen{\card{C_i}-2m_V\tau}\,.
\]
Finally, since $m>8$, our assumption gives
\[
    \card{C_i}>\frac{8n}{m}+3m_V\tau
    \geq\frac{8(n+m_V\tau)}m+2m_V\tau\,.
\]
Thus the gain strictly exceeds the closing cost, contradicting
local optimality.

\end{proof}

\subsection{Putting it Together}
We now prove \Cref{thm:private-evolution-clustering}.
\begin{proof}
Condition on the generator and coordinate-wise noise events in each
iteration.
The warm-start condition gives a synthetic point within
distance $\nfrac R8$ of each cluster initially. Inductively,
the generator bound above and \Cref{lem:localSearch:retainClusters}
show that this property persists after each iteration.
For $t\in[T]$, let $S^{(t-1)}$ denote the synthetic
dataset from iteration $t-1$, and let $V=\VarAPI(S^{(t-1)})$.
By \Cref{lem:localSearch:retainClusters},
the selected dataset contains a candidate within distance $\nfrac R8$
of each cluster. Furthermore, any selected candidate corresponding
to cluster $C_i$, say $y_i$, satisfies
$\rho(x_i,y_i)=\rho(x_i,V)$ for some $x_i\in C_i$ that voted for it.
In other words, for any $x\in C_i$,
\begin{align*}
    \rho(x,S^{(t)})
    &\leq \rho(x,x_i)+\rho(x_i,y_i) \\
    &\leq r+\rho(x_i,V) \\
    &\leq r+\max\inbrace{r,
        \inparen{1-\frac{c}{50\lambda}}\rho(x_i,S^{(t-1)})}\,.
\end{align*}
Hence
\[
    \max_{x\in D}\rho(x,S^{(t)})
    \leq\max\inbrace{2r,\,
        r+\inparen{1-\frac{c}{50\lambda}}
        \max_{x\in D}\rho(x,S^{(t-1)})}\,.
\]
Since $50r\lambda/c\geq2r$, iterating this bound and using the
warm-start condition yields
\[
    \max_{x\in D}\rho(x,S^{(t)})
    \leq\frac{50r\lambda}{c}
       +e^{-\frac{ct}{50\lambda}}R\,.
\]
Taking
$T\geq\left\lceil\frac{50\lambda}{c}
\log\inparen{\frac Rr}\right\rceil$ yields
\[
    \max_{x\in D}\rho(x,S^{(T)})
    \leq\frac{50r\lambda}{c}+r
    \leq\frac{100r\lambda}{c}\,.
\]
Each of the two events fails with probability at most
$\nfrac\beta{3T}$ conditional on preceding successful iterations.
Taking a union bound over the $2T$ failure events yields the desired
failure probability.
\end{proof}

\section{Synthetic Data Experiments}\label{apx:synthetic-data-experiments}
We consider the following simple model for PE in 2D Euclidean space.
\begin{enumerate}[1),noitemsep,topsep=0em]
    \item A well-clustered dataset of size $n=2000$ is generated from the Gaussian mixture model
    \[
        0.9 \cdot \mcal N([-1, 0], 0.1^2 I_2) + 0.1\cdot \mcal N([1, 0], 0.1^2 I_2)\,.
    \]
    \item We initialize PE at $m=3$ points $[0, -2], [0, 0], [0, 2]$.
    \item $\VarAPI(y, s^2)$ is modeled as a standard Gaussian distribution, and we generate 2 variations per synthetic data point at each generation with fixed $s^2 = 0.1^2$.
    \item After computing the noisy nearest neighbor histogram, we use either ranking, subsampling, or clustering to select $m=3$ points for the next iteration.
    \item We run PE for 20 iterations and depict the trajectory of the $m=3$ synthetic data points throughout.
\end{enumerate}

\paragraph{Ranking Fails.}
We illustrate the trajectory of PE with ranking in \Cref{fig:ranking-fails}.
Noticeably, the algorithm immediately drops all points near the smaller cluster.
From there, it never recovers.
\begin{figure}[htbp]
    \centering
    \includegraphics[width=\linewidth]{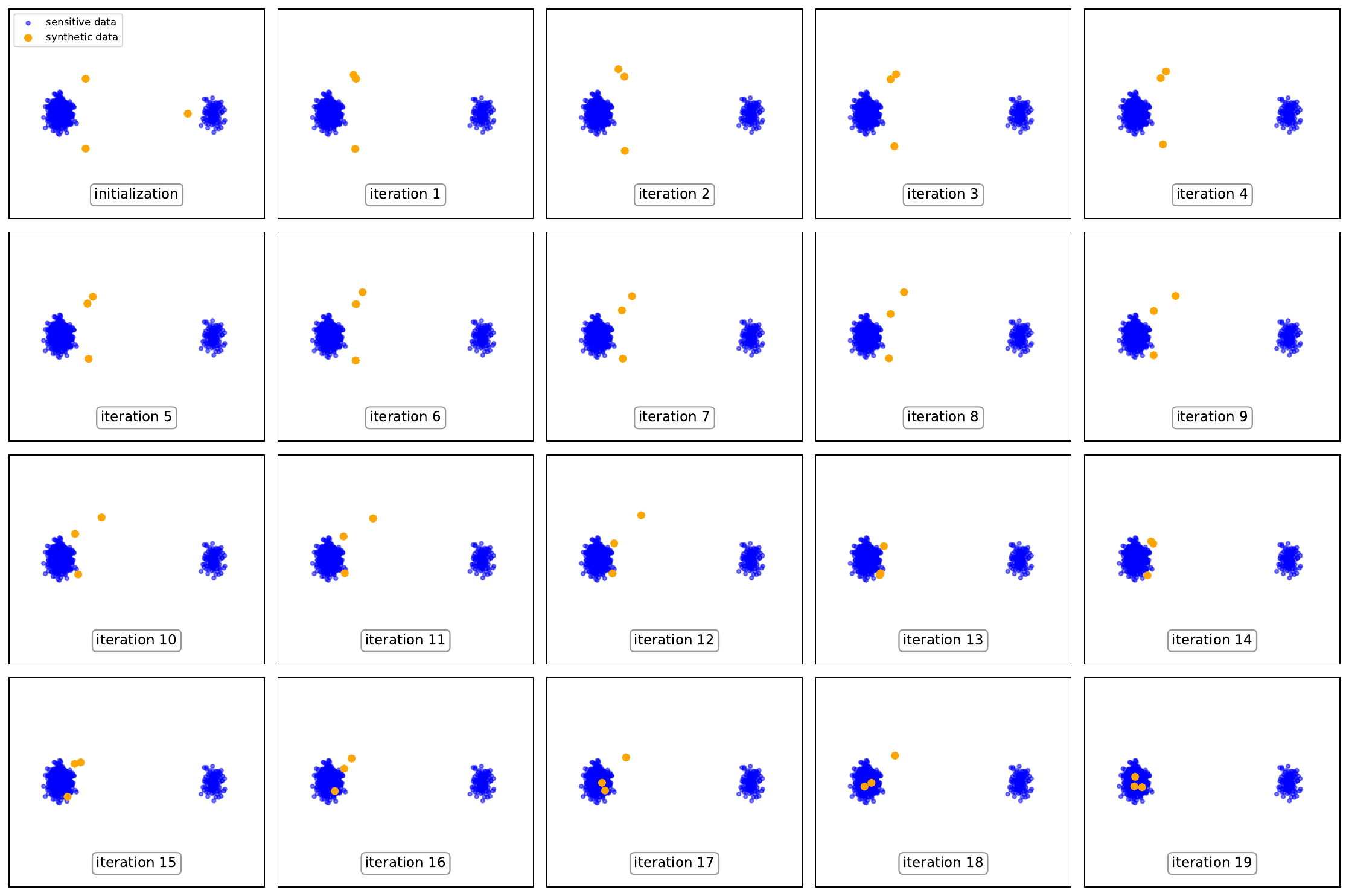}
    \caption{The trajectory of PE with ranking selection (\augPE) on generated 2D Euclidean data.}
    \label{fig:ranking-fails}
\end{figure}

\paragraph{Subsampling Fails.}
The trajectory of PE with subsampling is depicted in \Cref{fig:subsampling-fails}.
We notice that subsampling also drops points near the smaller cluster, and also does not quite converge even to the bigger cluster.
\begin{figure}[hbtp]
    \centering
    \includegraphics[width=\linewidth]{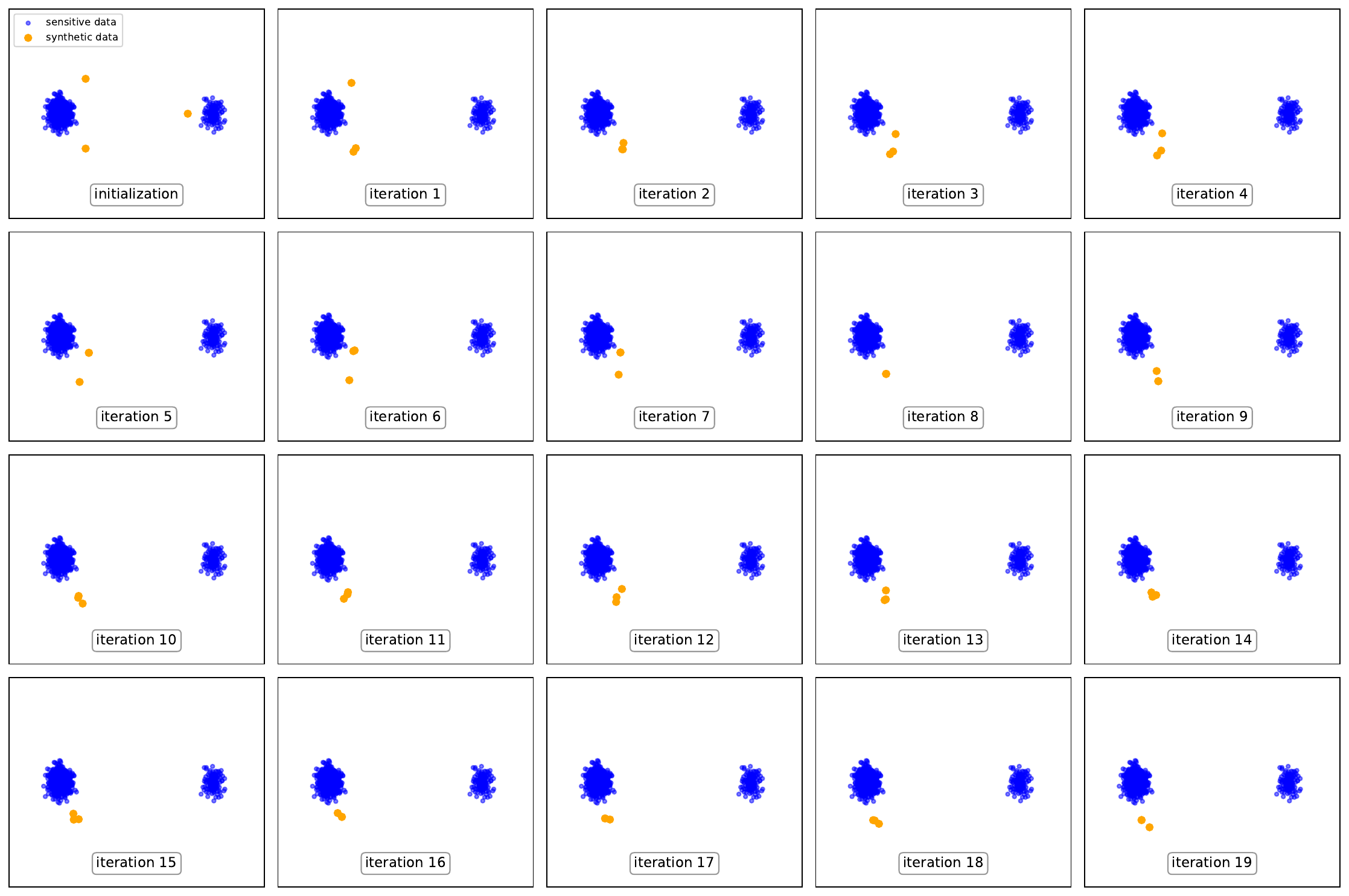}
    \caption{The trajectory of PE with subsampling selection (\subsamplePE) on generated 2D Euclidean data.}
    \label{fig:subsampling-fails}
\end{figure}

\paragraph{Clustering Succeeds.}
Finally, we implement a simplified version of \GAPE\ using the sklearn implementation\footnote{\url{https://scikit-learn-extra.readthedocs.io/en/stable/generated/sklearn_extra.cluster.KMedoids.html}} of the $k$-medoids clustering~\citep{maranzana1963location,park2009simple} as a selection mechanism.
The trajectory for this algorithm is depicted in \Cref{fig:clustering-succeeds}.
We observe that the algorithm consistently selects at least one point near each cluster, even though the location of the third point may oscillate between the two clusters.
\begin{figure}[htbp]
    \centering
    \includegraphics[width=\linewidth]{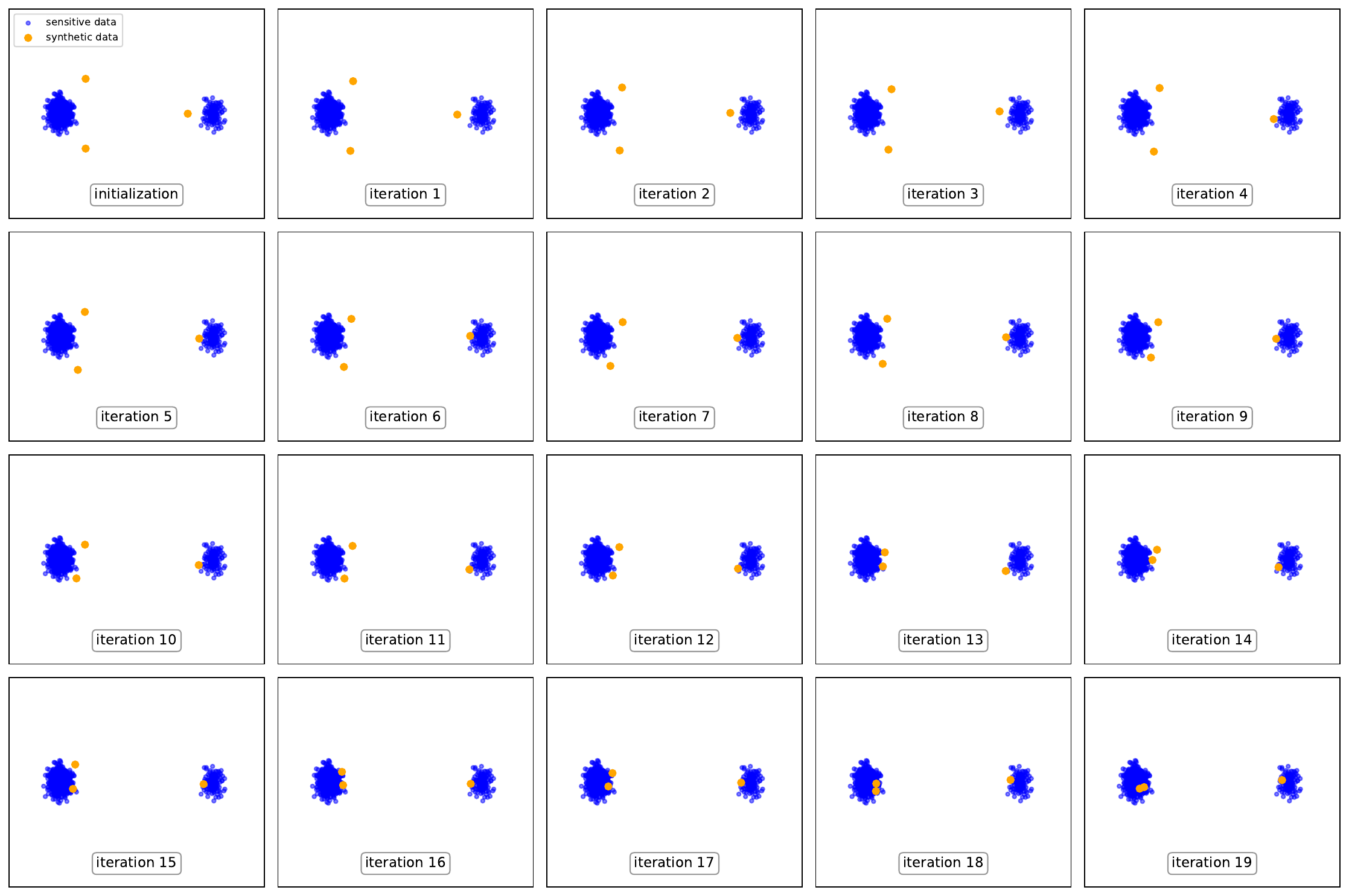}
    \caption{The trajectory of our proposed variant of PE (\GAPE) on generated 2D Euclidean data.}
    \label{fig:clustering-succeeds}
\end{figure}

\section{Additional Experimental Details}\label{apx:experimental-details}
\subsection{Hyperparameter Tuning} 
All of the hyperparameters except those specific to \GAPE\ are common. Thus, we expect any improvement by tuning these parameters to be roughly shared by all three mechanisms. For PubMed, we fix these parameters to be those that appear in the main experiments~\cite{xie2024text}. That is, $m=2000$, $T=10$, and number of variations of each synthetic point per iteration is 7. The number of initial calls to $\RanAPI$ is $2m$. For Reddit, we use $m=10{,}000$ synthetic samples.
In all experiments, the parameter $R$ is set by performing kmeans with $k=m$ on the dataset, computing the maximum distance from a true data point to its assigned center then multiplying this distance by 10. Further hyperparameter tuning of \GAPE\ could result in better relative performance.

\subsection{Precision \& Recall}
A popular implementation of the precision metric\footnote{\url{https://github.com/marcojira/fld/blob/main/fld/metrics/PrecisionRecall.py}} by \citet{jiralerspong2023feature} is based on forming a small ball of radius $r_x > 0$ about each point in the dataset $x\in D$ and checking the fraction of synthetic data points that fall into at least one ball.
$r_x$ is defined to be the distance to the $k$-th nearest neighbor of $x$ in $D$ for some small $k$ (\eg{}, $k=3, 4$).

On the other hand, the recall metric simply exchanges the operation by computing nearest neighbor radius balls around each synthetic data point $s\in S$ and asking for the fraction of true data points $x\in D$ that fall in some ball.

While the nearest neighbor distance of the dataset needed for precision is determined only by the data, the same cannot be said for the nearest neighbor distance of the synthetic dataset, which is controlled by the algorithm.
For example, the trivial algorithm that always outputs two points extremely far away will have perfect recall under this metric!

Thus, we modify the implementation of recall to be data-based: we form the nearest neighbor radius ball about each point in the \emph{dataset} $x\in D$ and compute the fraction of balls that intersect the synthetic dataset.
As a sanity check, the trivial algorithm above should no longer be able to ``game'' this metric.

We emphasize that this modification is consistent when evaluating all algorithms for a fair comparison. Further, note that the metric as defined is somewhat tuned for the case when $m$ is close to $n$, as the value $k$ is small. In settings where $m \ll n$, the recall is likely to be pretty small, and thus other metrics may be more appropriate measures.

\subsection{API Prompts}

\subsubsection{PubMed}

$\RanAPI$: ``Using a variety of sentence structures, write an abstract for a medical research paper:''

$\VarAPI$: ``Please rephrase the following sentences \{tone\} as an abstract for medical research paper: \{sample\}"

  The \{tone\} parameter is randomly selected from these options:
  \begin{itemize}[noitemsep,topsep=0em]
      \item ``in a professional way''
      \item ``in a professional tone''
      \item ``in a professional style''
      \item ``in a concise manner''
      \item ``in a creative style''
      \item ``using imagination''
      \item ``in a storytelling tone''
      \item ``in a formal manner''
      \item ``using a variety of sentence structures''
  \end{itemize}
  The \{sample\} is replaced with the source text to be rephrased.

\subsubsection{Reddit}

$\RanAPI$: ``Given the category, you are required to provide a succinct Reddit post title in 30-100 words.''

$\VarAPI$: ``Please rephrase the following sentences \{tone\} as a Reddit poster:
  \{sample\}''

  The \{tone\} parameter is randomly selected from:
  \begin{itemize}[noitemsep,topsep=0em]
      \item ``in a detailed way''
      \item ``in a professional way''
      \item ``with more details''
      \item ``with a professional tone''
      \item ``in a professional style''
      \item ``in a concise manner''
  \end{itemize}

  \subsection{Further Results}\label{furtherresults}

  \begin{figure}[htbp]
    \centering
    \begin{subfigure}{0.45\textwidth}
        \centering
        \includegraphics[width=\textwidth]{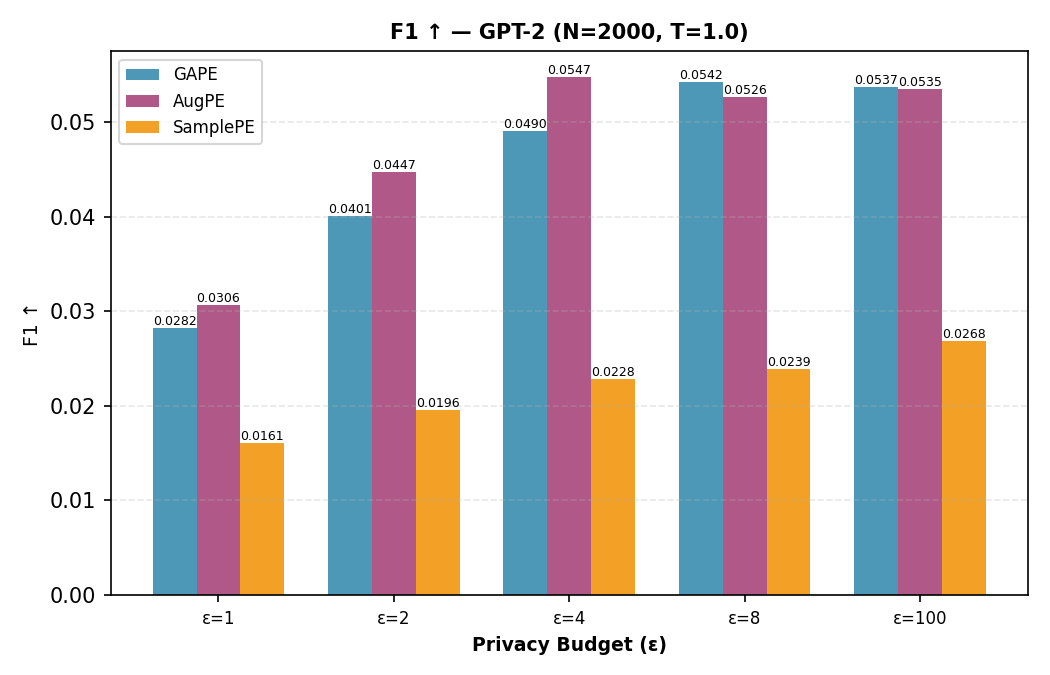}
    \end{subfigure}
    \hfill
    \begin{subfigure}{0.45\textwidth}
        \centering
        \includegraphics[width=\textwidth]{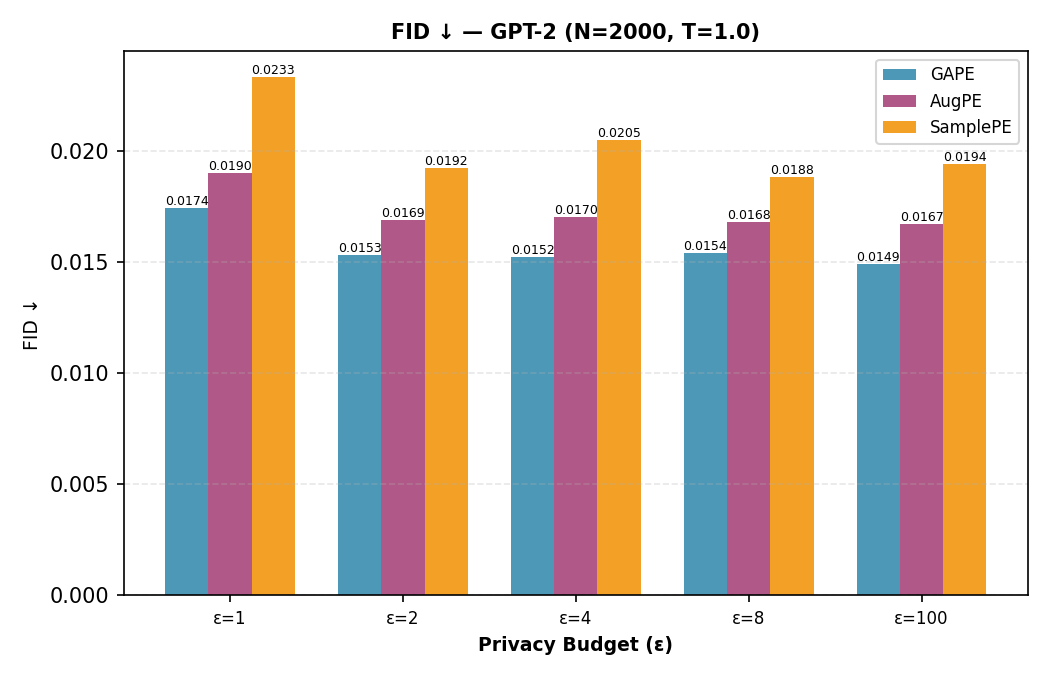}
    \end{subfigure}
        \hfill
    \begin{subfigure}{0.45\textwidth}
        \centering
        \includegraphics[width=\textwidth]{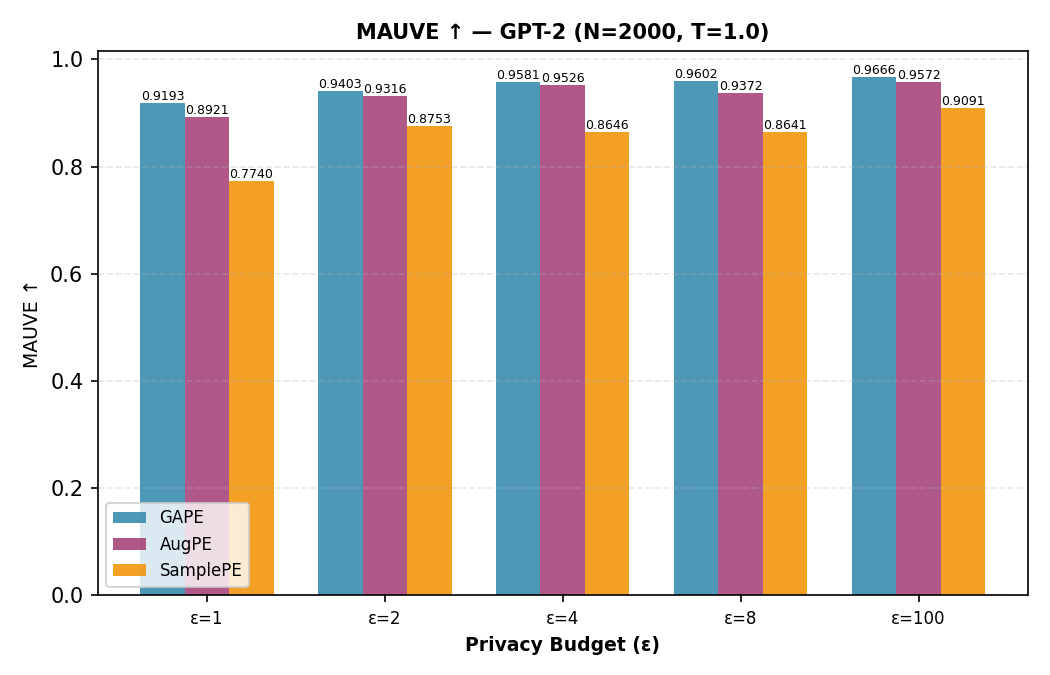}
    \end{subfigure}
            \hfill
    \begin{subfigure}{0.45\textwidth}
        \centering
        \includegraphics[width=\textwidth]{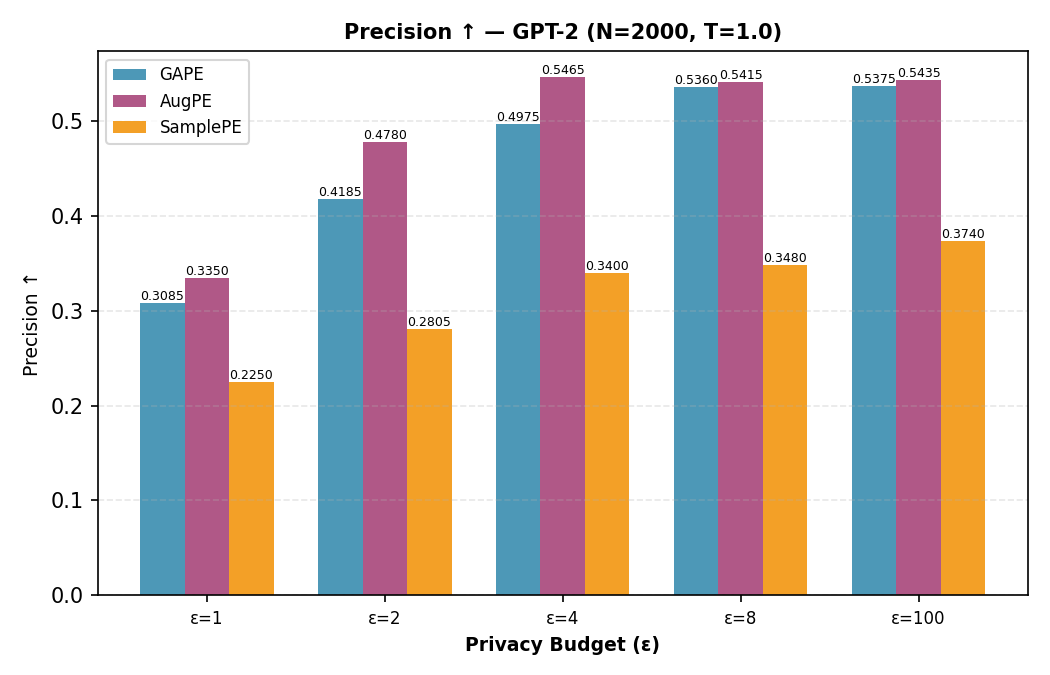}
    \end{subfigure}
    \caption{PubMed, $m=2000$}
    \label{fig:pubmed-additional-metrics}
\end{figure}

  \begin{figure}[htbp]
    \centering
    \begin{subfigure}{\textwidth}
        \centering
        \includegraphics[width=\linewidth,height=0.21\textheight,keepaspectratio]{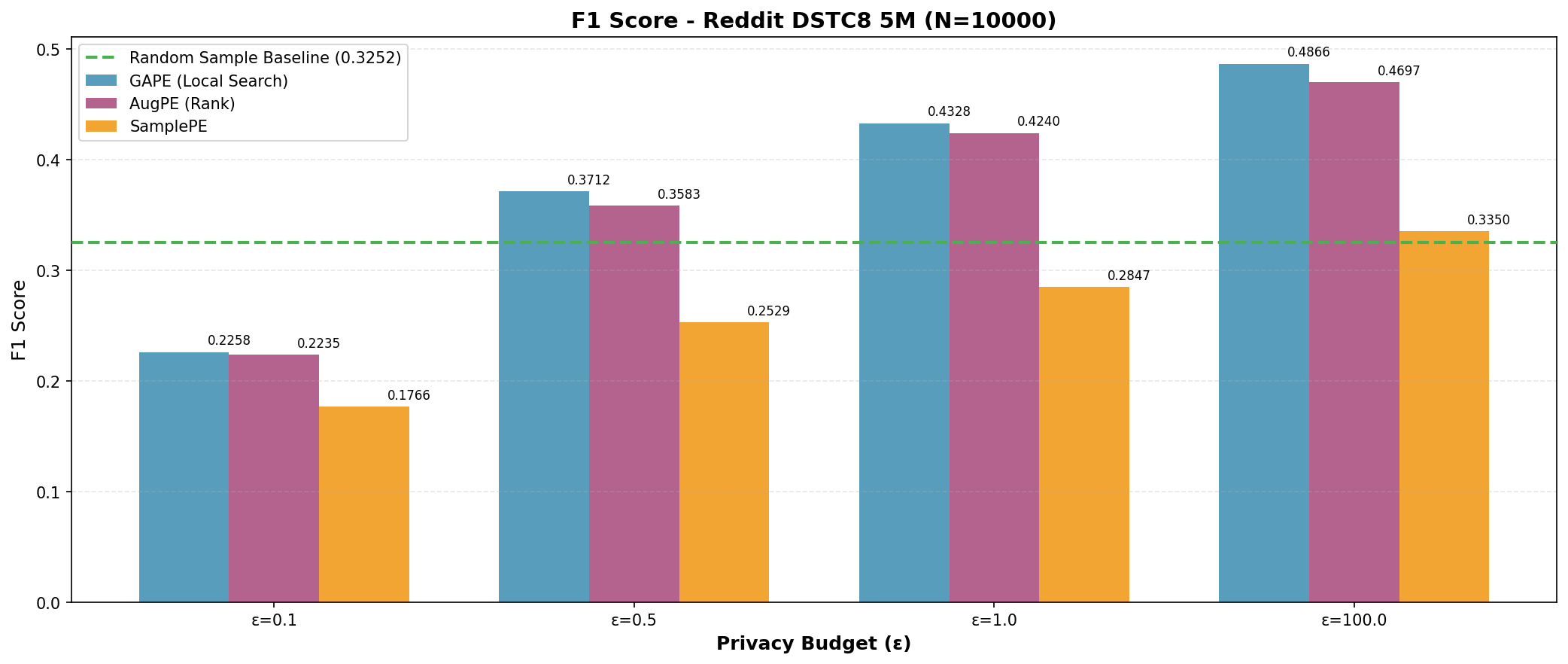}
    \end{subfigure}
    \begin{subfigure}{\textwidth}
        \centering
        \includegraphics[width=\linewidth,height=0.21\textheight,keepaspectratio]{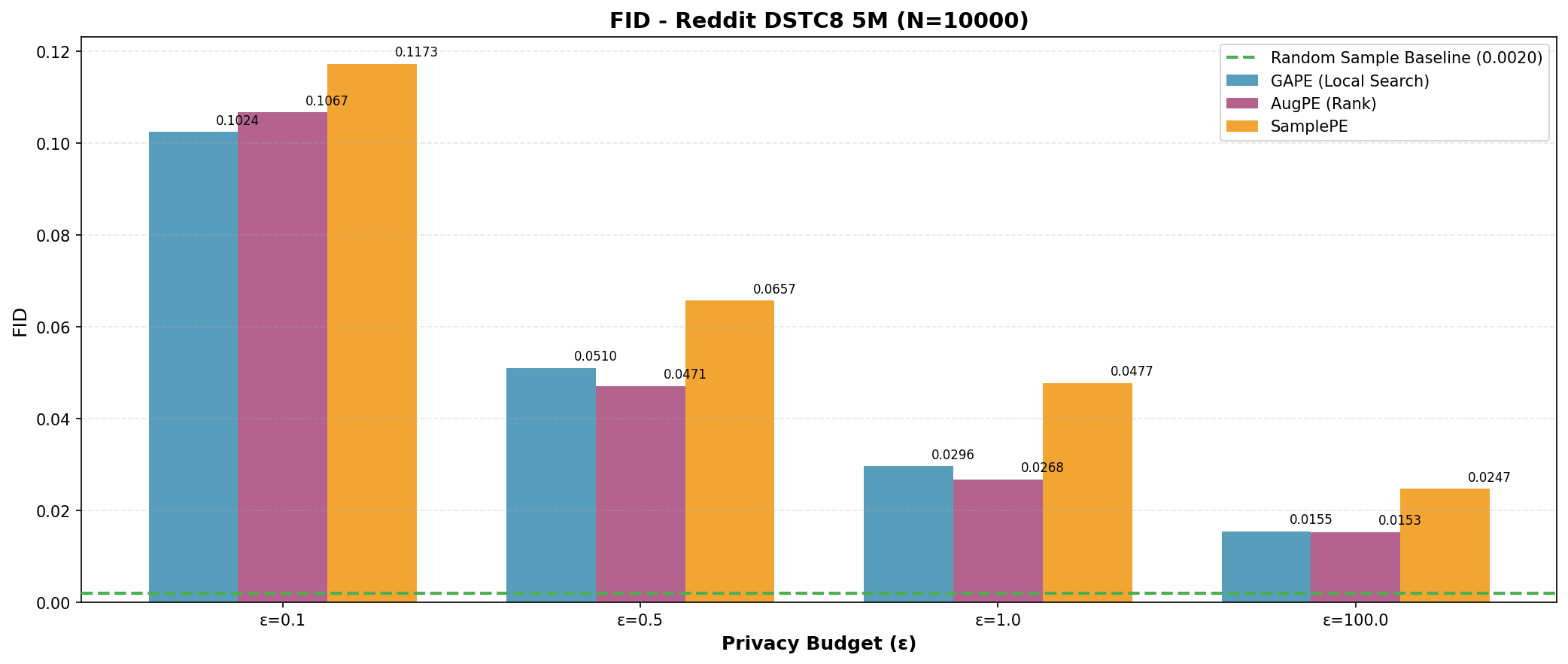}
    \end{subfigure}
    \begin{subfigure}{\textwidth}
        \centering
        \includegraphics[width=\linewidth,height=0.21\textheight,keepaspectratio]{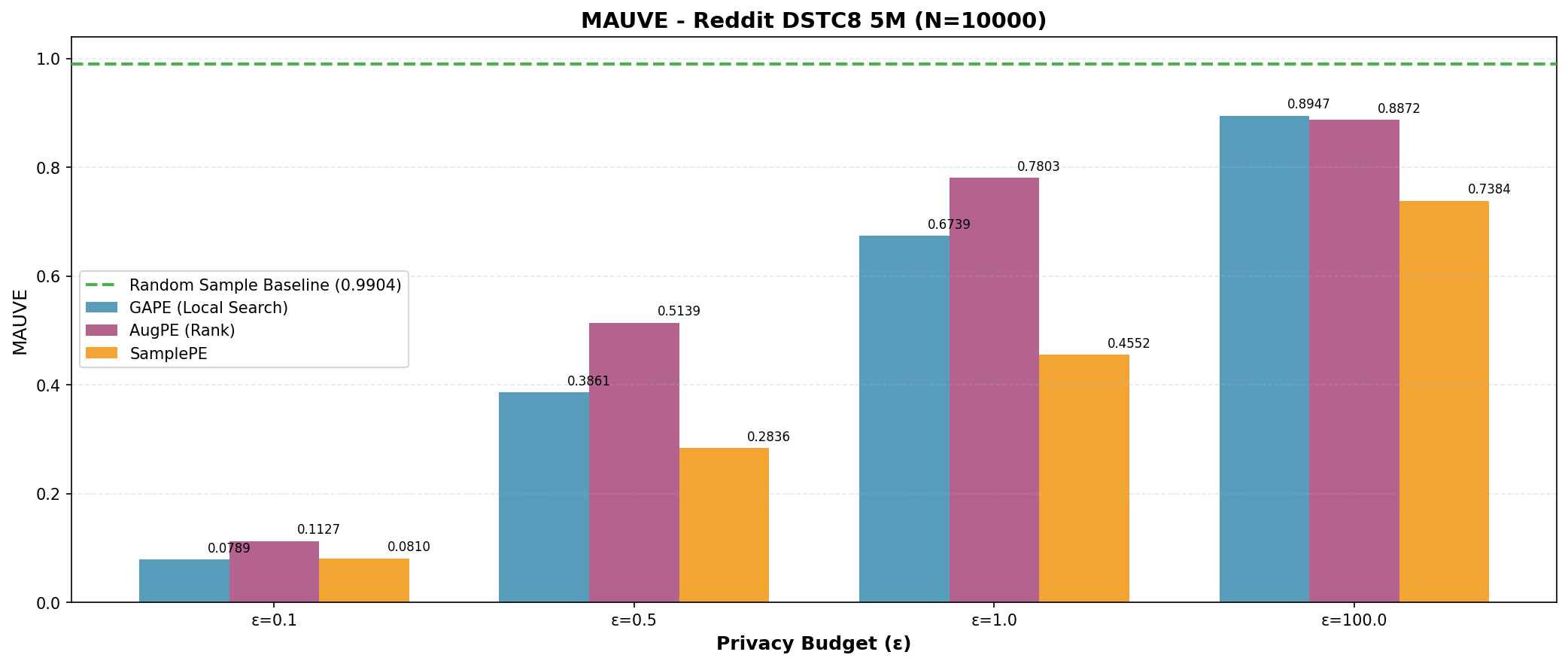}
    \end{subfigure}
    \begin{subfigure}{\textwidth}
        \centering
        \includegraphics[width=\textwidth,height=0.21\textheight,keepaspectratio]{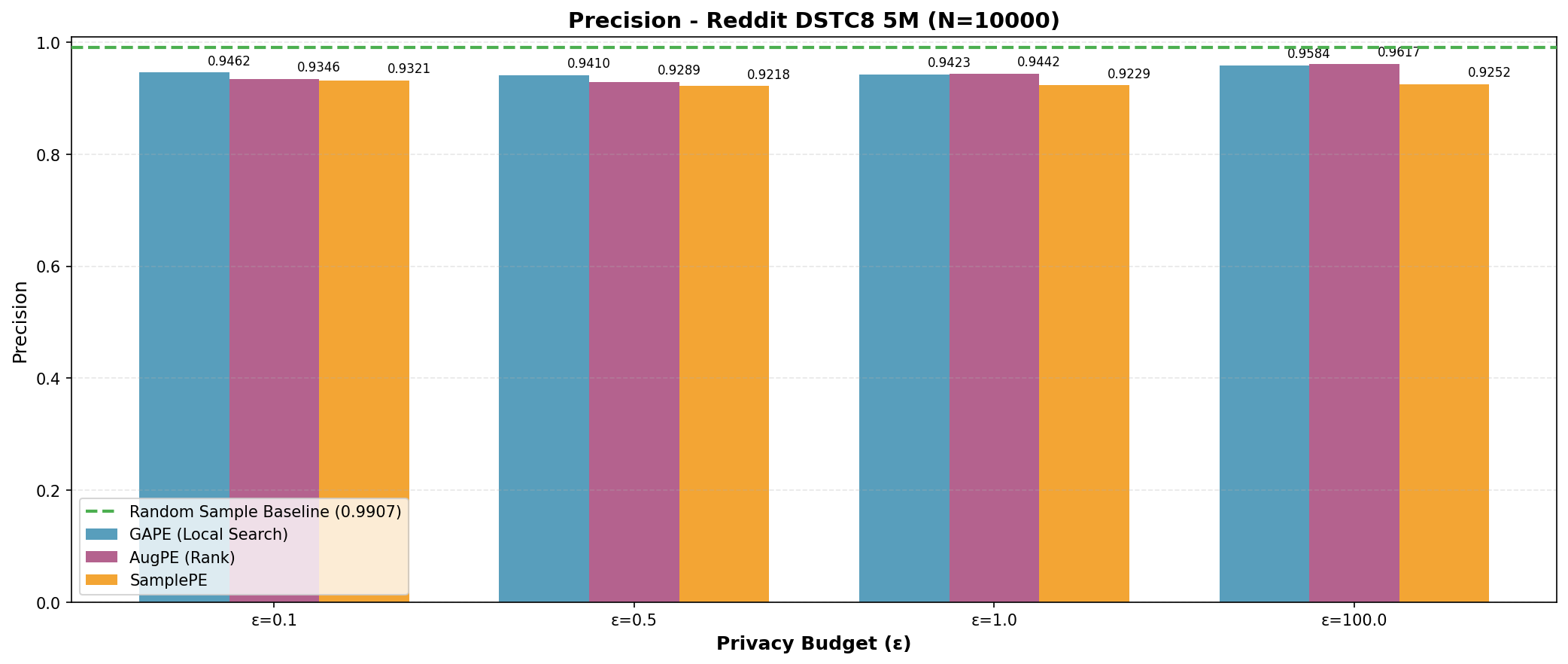}
    \end{subfigure}
    \caption{Reddit, $m=10,000$}
    \label{fig:reddit-additional-metrics}
\end{figure}

\begin{table}[htb]
    \centering
    \small
    \caption{Recall for a wider range of $\eps$ on PubMed ($m=2000$) and Reddit ($m=10{,}000$).}
    \label{tab:recall-eps}
    \begin{tabular}{lccccccc}
        \toprule
        Method & $\eps=0.1$ & $\eps=0.5$ & $\eps=1$ & $\eps=2$ & $\eps=4$ & $\eps=8$ & $\eps=100$ \\
        \midrule
        \multicolumn{8}{l}{\emph{PubMed}} \\
        \augPE & 0.0040 & 0.0105 & 0.0161 & 0.0234 & 0.0288 & 0.0276 & 0.0281 \\
        \GAPE & 0.0036 & 0.0104 & 0.0148 & 0.0211 & 0.0258 & 0.0286 & 0.0282 \\
        \subsamplePE & 0.0032 & 0.0062 & 0.0083 & 0.0102 & 0.0118 & 0.0124 & 0.0139 \\
        \midrule
        \multicolumn{8}{l}{\emph{Reddit}} \\
        \augPE & 0.1269 & 0.2220 & 0.2734 & 0.3216 & 0.3320 & 0.3354 & 0.3107 \\
        \GAPE & 0.1282 & 0.2312 & 0.2809 & 0.3204 & 0.3246 & 0.3241 & 0.3260 \\
        \subsamplePE & 0.0975 & 0.1465 & 0.1683 & 0.1982 & 0.2086 & 0.2156 & 0.2045 \\
        Random Sample Baseline & -- & -- & -- & -- & -- & -- & 0.1946 \\
        \bottomrule
    \end{tabular}
\end{table}

\end{document}